%% file: Main.tex
\pdfoutput=1

\documentclass[12pt]{book}
\usepackage{iiit_thesis}
\usepackage{times}
\usepackage{epsfig}
\usepackage{amsmath}
\usepackage{amssymb}
\usepackage{latexsym}
\usepackage{multirow}
\usepackage{algorithm2e}
\usepackage{graphics}
\usepackage{soul}
\usepackage{etoolbox}
\usepackage{multicol}
\usepackage{graphicx}
\usepackage{caption}
\usepackage{subcaption}
\usepackage{booktabs}
\usepackage{tabularx}
\usepackage{amssymb}
\usepackage{textcomp}
\usepackage{gensymb}
\usepackage{multirow}
\usepackage{hyperref}

\newtheorem{assum}{Assumption}
\usepackage[table]{xcolor}
\graphicspath{{images/}{./}} 
\usepackage{amssymb}
\long\def\symbolfootnote[#1]#2{\begingroup%
\def\thefootnote{\fnsymbol{footnote}}\footnote[#1]{#2}\endgroup}
\renewcommand{\baselinestretch}{1.2}
\onecolumn
\begin{document}

\input{titlePage.tex}

\newpage
\thispagestyle{empty}
\renewcommand{\thesisdedication}{{\large Copyright \copyright~~Swati Dantu, 2023\\}{\large All Rights Reserved\\}}
\thesisdedicationpage

\input{certificate.tex}

\newpage
\thispagestyle{empty}
\renewcommand{\thesisdedication}{\large To My Family}
\thesisdedicationpage

\mastersthesis

\renewcommand{\baselinestretch}{1.5}
\newtheorem{theorem}{Theorem}
\newtheorem{lem}{Lemma}
\newtheorem{remark}{Remark}
\newtheorem{assumption}{Assumption}
\newtheorem{proof}{Proof}


\counterwithin{remark}{chapter}
\counterwithin{lem}{chapter}
\counterwithin{assumption}{chapter}
\counterwithin{theorem}{chapter}


\input{01_acknowledgements}


\chapter*{Abstract}
\label{ch:abstract}
\input{00_abstract}
\tableofcontents
\listoffigures
\listoftables

\input{02_introduction}
\input{03_Bipedal}
\input{04_Quadrotor_new}
\input{05_Conclusion}

\bibliographystyle{IEEEtran}
\bibliography{ref} 

\end{document}

%% file: titlePage.tex
\thispagestyle{empty}
\begin{center}
{\Large \bf Adaptive Artificial Time Delay Control for Robotic Systems}

\vspace*{2.75cm}
{\large Thesis submitted in partial fulfillment\\}
{\large  of the requirements for the degree of \\}

\vspace*{1cm}
{\it {\large Master of Science in \\ \textbf{Electronics and Communication Engineering} \\ by Research}}

\vspace*{1cm}
{\large by}

\vspace*{5mm}
{\large Swati Dantu\\}
{\large 2020702014\\
{\small \tt swati.dantu@research.iiit.ac.in}}

\vspace*{3.0cm}
{\large International Institute of Information Technology\\}
{\large Hyderabad - 500 032, INDIA\\}
{\large June 2023\\}
\end{center}

%% file: certificate.tex
\newpage
\thispagestyle{empty}
\vspace*{1.5cm}
\begin{center}
{\Large International Institute of Information Technology\\}
{\Large Hyderabad, India\\}
\vspace*{3cm}
{\Large \bf CERTIFICATE\\}
\vspace*{1cm}
\noindent
\end{center}
It is certified that the work contained in this thesis, titled "\textit{Adaptive Artificial Time Delay Control for Robotic Systems}" by Swati Dantu, has been carried out under
my supervision and is not submitted elsewhere for a degree.

\vspace*{3cm}
\begin{tabular}{cc}
\underline{\makebox[1in]{}} & \hspace*{5cm} \underline{\makebox[2.5in]{}} \\
Date & \hspace*{5cm} Adviser: Dr. Spandan Roy
\end{tabular}
\oneandhalfspace

%% file: 01_acknowledgements.tex
\chapter*{Acknowledgments}
\label{ch:ack}

First and foremost I would like to express my sincerest and unparalleled gratitude to Dr. Spandan Roy for trusting me and giving me an opportunity to be a part of Robotics Research Center. It is because of this opportunity, I was able to learn to conduct research at the highest level. He always has been very open to my ideas and was very kind in helping me with subject matter: be it fundamentals or one of my ill-conceived ideas.

I feel immensely fortunate to be a part of IIIT Hyderabad. I found incredible people here that I am proud to call friends with whom I formed unbreakable bonds forged during the trying times of Covid-19 pandemic. I had the best of times locked in the campus with these amazing people: Jhanvi, Kaustab, Rishabh, Sasi, Girish, Rahul, Sudarshan, Omama, Ayyappa, Bhanu, Ananth and many more students at IIIT. I am particularly grateful for my seniors Viswa and Udit for mentoring me. I am ever grateful to Rishabh for being an amazing project partner. 

Finally but most importantly, I would like to thank my parents and sister for always being there for me and believing in me. I would like to specifically thank my sister for always being there for me with her infectious positive energy that made any day brighter. I am extremely fortunate to always have them by my side.

Thank you all!

%% file: 00_abstract.tex
\vspace{-5mm}
{Artificial time delay controller} was conceptualised for nonlinear systems to reduce dependency on precise system modelling unlike the conventional adaptive and robust control strategies. In this approach \textit{unknown dynamics is compensated by using input and state measurements collected at immediate past time instant (i.e., artificially delayed)}. The advantage of this kind of approach lies in its simplicity and ease of implementation. However, the applications of artificial time delay controllers in robotics, which are also robust against unknown state-dependent uncertainty, are still missing at large. This thesis presents the study of this control approach toward two important classes of robotic systems, namely a fully actuated bipedal walking robot and an underactuated quadrotor system.

In the first work, we explore the idea of a unified control design instead of multiple controllers for different walking phases in adaptive bipedal walking control (i.e. control taking care of unknown robot parameters) while bypassing computing constraint forces, since they often lead to complex designs. State-of-the art attempts to design a single controller for all walking phases either ignored or oversimplified the state-dependent constraint forces which may lead to conservative performance or even instability. This work proposes an innovative artificial time delay based adaptive control method, which covers the entire bipedal walking phase and provides robustness against state-dependent unmodelled dynamics such as constraint forces and external impulsive forces arising during walking. Studies using a high fidelity simulator under various forms of disturbances show the effectiveness of the proposed design over the state of the art.

The second work focuses on quadrotors employed for applications such as payload delivery, inspection and search-and-rescue. Such operations pose considerable control challenges, especially when various (a priori unbounded) state-dependent unknown dynamics arises from payload variations, aerodynamic effects and from reaction forces while operating close to the ground or in a confined space. 
The existing adaptive control strategies for quadrotors, unfortunately, are suitable to handle unknown state-dependent uncertainties. We address such unsolved control challenge in this work via a novel adaptive artificial time delay controller. 
The effectiveness of this controller is validated using experimental results. 

%% file: 02_introduction.tex
\chapter{Introduction}
\label{ch:intro}

Owing to the advancement in technology, there is a steady rise in the use of robotics in both research and industry for commercial purposes. These robots are often required to be operated autonomously to perform repeated tasks more efficiently than humans. One of the critical components to achieve autonomy is control system. Control has a broad classification from choosing a high-level trajectory tracking to a low-level control for choosing actuator inputs. The primary objective of a high level controller is to calculate desired actuator inputs for a specific objective such a trajectory tracking or velocity tracking. These inputs are further converted to current or voltage inputs by another controller at the lower level.

In the following sections of this introductory chapter, the motivation for this thesis is discussed, followed by a brief overview of the relevant fundamental principles and methods involved in constructing this work. Then, a brief section on the contribution of this thesis is provided, followed by the overall organization of the thesis.

\section{Motivation}

A nonlinear system in a practical scenario is always subjected to parametric and non parametric uncertainties making the control problem extremely challenging. To tackle this challenge, researchers have extensively used adaptive and robust control strategies \cite{rob4rotor, SatARC, ARCsin, ARCrod, mu}. A robust control strategy provide robustness against system uncertainties within a predefined bound. Whereas, an adaptive control strategy estimates the unknown system parameters online without their a priori knowledge. This apparently gives adaptive controller an edge over a robust controller. However, compared to a robust controller, an adaptive controller requires structural knowledge of the system as well as is computationally intensive. 

In view of the individual challenges of a typical adaptive and robust controller, adaptive-robust control strategy was devised combining the best qualities of both the approaches (cf. \cite{ASMCmimo, ASMCmultibody, ASMCog, ASMCtwist, AMSCsupertwist} and references therein). In this approach, uncertainties are represented by one single lumped function and then it is estimated online. 
Nevertheless, these existing adaptive-robust controllers a priori knowledge of the nominal values of uncertain system parameters. This implies that the accurate modelling of the system is necessary which is not possible in the face of unmodelled dynamics.

Therefore, the time-delay estimation (TDE) \cite{Ref:6-2, Hsia1991,TDC,TDCInternal, TDCmotioncontrol,TDCrobust, TDCterminalSM, TDCunderwater} technique has been introduced as a simple alternative for the design of a model free controller that addresses the lack of a priori knowledge of uncertainty and/or unmodelled dynamics. The TDE technique assumes a single lumped system uncertainty and estimates it with only control input and state information of the previous time instant. This time-delay control (TDC) usually consists of two parts. One part cancels the unknown and non-linear dynamics of the robot and the other part is the desired dynamics injection part.

The estimation error originating from the TDE method, a.k.a the TDE error, deteriorates the control performance. Various robust and adaptive control strategies have been employed along with TDE to mitigate such effect \cite{Ref:new1, Ref:33,Ref:34,TDCmotioncontrol, Ref:lee1, Lee_humanoid_49, Ref:eh, Ref:fuel,Ref:sm}. However, these existing designs rely on a priori boundedness of the TDE error, which ignores state-dependent uncertainty. This thesis primarily addresses such a gap in literature in the context of two important classes of Euler-Lagrange systems.

\section{Preliminaries}
\subsection{Choosing Euler-Lagrange Systems}

Any electro-mechanical system that captute a wide range of real-world robotic systems such as mobile robots, aerial robots, legged robots, manipulators \cite{ELadap, ELselfconfig, ELvanish, TDCrobust, TDCunderwater, wang2022adaptive, ye2020switching} can be modelled using Euler-Lagrange dynamics. Usually in modelling a system with lumped parameters such a electro-mechanical systems, there are two approaches: 
\begin{itemize}
    \item Derivation of equations of motion using the first principles of laws of forces such as Newton's second law and Kirchhoff's law. Such approach becomes tedious for complicated systems and require significant domain knowledge.
    \item Derivation of system equations called the EL equations which are non-linear differential equations whose starting point is the definition of energy functions which leads to the definition of Lagrangian. This approach provides a generalised representation of system dynamics eliminating the need to identify different kinds of plants for each controller.
\end{itemize}

\subsubsection{Overview of Euler-Lagrange representation}

In this representation, the equations of motion are formulated by minimising Lagrangian Integral defined as:

\[
    \int_{t_i}^{t_f} \mathcal{L}(t,q(t),\dot{q}(t))dt
\]
where $q(t),\dot{q}(t) \in \mathbb{R}^n$ are the generalised coordinated and their corresponding velocities. Considering $q(t_i)$ and $q(t_f)$ to be start and end points of a path taken according to the Lagrangian, 
Hamilton's principle states that the path must have stationary action \cite{EL_base}. Using this principle and considering non-conservative forces such as control input, external disturbance and friction in mechanical systems:
\begin{equation}
    \frac{d}{dt}\left(\frac{\partial \mathcal{L}}{\partial \dot{q}} \right) - \frac{\partial \mathcal{L}}{\partial {q}} = \tau + d_s - h(q,\dot{q}) 
\end{equation}
where $\tau \in \mathbb{R}^n$ is the control input, $d_s \in \mathbb{R}^n$ is the lumped parameter of all external disturbances and $h(q,\dot{q}) \in \mathbb{R}^n $ is the parameter representing the frictional forces.
The above mentioned EL system can be compactly can be written as:

\begin{equation}\label{EL_fin}
        M(q)\ddot {q}+ C(q ,\dot{q})\dot{q}+ G(q)+ f(\dot{q})+d_s =  \tau
    \end{equation}
    where $M(q)\in \mathbb {R}^{n \times n}$ is the generalised inertia matrix which is uniformly positive definite,
    $C(q ,\dot{q}) \in \mathbb{R}^{n \times n}$ denotes the centripetal and coriolis terms, 
  $G(q) \in \mathbb{R}^n$ denotes the gravity vector
    $f(\dot{q}) \in \mathbb{R}^n$ represents the frictional terms.
    
The above representation (\ref{EL_fin}) has some properties that are exploited for control design (cf.\cite{TDCtextbook, roy2019simultaneous, roy2021adaptive, roy2020towards, baldi2020towards}). 
\begin{itemize}
    \item $M(q),g(q) ~ \text{and} ~ d_s$ are bounded individually.
    \item $C(q ,\dot{q})$ is upper bounded as $\lVert C(q ,\dot{q}) \rVert \leq C_b \lVert \dot{q} \rVert ~ \text{where} ~ C_b \in \mathbb{R}^+$
    \item $f(\dot{q})$ is upper bounded as $\lVert f(\dot{q}) \rVert \leq f_b \lVert \dot{q} \rVert ~ \text{where} ~ f_b \in \mathbb{R}^+$
\end{itemize}

These properties are later used in Chapters 2 and 3 for control design.
\subsection{Stability Notion}
For an autonomous nonlinear system that is represented by the following dynamic equation(cf.\cite{khalil})
\begin{equation}
    \dot{x} = f(x(t)),~~ x(0) = x_0 \nonumber
\end{equation}
with $x(t) \in \mathbb{R}^n$ representing the state vector, $f$ satisfying the standard conditions for existence and uniqueness such a \textit{Lipschitz continuous} with respect to $f$; the stability is attained when the function $f$ at equilibrium point $x = x_e$ is $f(x_e) = 0$ and $f'(x_e)=0$.
This equilibrium point $x_e$ is said to be:
\begin{itemize}
    \item \textbf{Stable} in the sense of Lyapunov, if for any $\epsilon > 0$, there exists $\delta > 0$ such that
    \begin{equation}
        \lVert x(0) - x_e \rVert < \delta \implies 
        \lVert x(t) - x_e \rVert < \epsilon ~~~~
        \forall t \geq 0 \nonumber
    \end{equation}
Stability ensures that the states starting within the bound $\delta$, will remain within the bound $\epsilon$   

\item \textbf{Asymtodically Stable} in the sense of Lyapunov if 
\begin{itemize}
    \item The equilibrium point $x_e$ is stable
    \item The equilibrium point $x_e$ is locally attractive i.e.,
\[
    \lim_{t\to\infty} \lVert x(t) - x_e \rVert = 0
\]

\end{itemize}
This implies that the system is not only stable but also is converging towards the equilibrium point $x_e$.

\item \textbf{Exponentially Stable} in the sense of Lyapunov if there exist constants $\alpha,\beta > 0,\epsilon > 0$ such that
\begin{equation}
    \lVert x(t) - x_e \rVert \leq \alpha \lVert x(0) - x_e \rVert e^{-\beta t} 
   ~~~~ \forall t \geq 0 ~ \text{and}~ \lVert x(0) \rVert \leq \epsilon \nonumber
\end{equation}

Exponential stability ensures that the system converges to equilibrium at an exponential rate. 

\item \textbf{Uniformly Ultimately Bounded(UUB)} with an ultimate bound $b$ if  there exist $b >0, c>0$ and for every $a \in (0,c)$ there exist $T = T(a,b) \geq 0$ such that 
\begin{equation}
    \lVert x(0) \rVert \leq a \implies 
    \lVert x(t) \rVert \leq b, ~~~~
    \forall t \geq T \nonumber
\end{equation}
This is used to show the boundedness even if there is no equilibrium point at the origin. 
\end{itemize}

\subsection{Artificial Delay-based Controller: A Background on TDE method}

While there is existing literature that deals with systems that have explicit time-delays from various sources (such as actuator delay, communication channel-induced delay, etc.), a parallel body of literature is emerging where time-delay is artificially introduced into the system to achieve various advantages. In studies \cite{TDE_extreme_34,TDE_extreme_35}, external time delay is purposefully introduced into the system to provide state derivative feedback. Time-Delayed Control (TDC) \cite{Ref:6, Ref:6-2} was conceptualized as an alternative robust control paradigm, especially for uncertain nonlinear systems with the goal of reducing system modeling dependency. However, these studies mostly focus on linear systems.

Traditional robust control strategies require prior knowledge of the entire system to devise a suitable predefined uncertainty bound. However, it is not always practical to find such a bound. In contrast, TDC does not require a thorough understanding of the system model. This control methodology combines all uncertain system dynamics terms into a single lumped unknown function and approximates it using control input and state information from the most recent time instant, which is referred to as time-delayed estimation (TDE).

The benefits of TDC are well-known, as it greatly reduces the burden of tediously identifying a complex system model and is easy to implement. It is worth noting that researchers have used a Neural Network (NN)-based technique to approximate uncertain system dynamics \cite{adap_38,NN_40}. However, TDC has a lower computational cost than NN-based approaches, as demonstrated in \cite{NN_41}. This is because the latter involves a greater number of tunable parameters and requires expert knowledge during the approximation stage. TDE method approximates the lumped system uncertainty by only using control input and state information from the immediate past time instant. Moreover, the design process does not require expert knowledge \cite{Ref:6, Ref:6-2, TDC_low_freq_42, TDC_tiltRotor_43, roy2015robust}.

\subsubsection{A Brief Outline of TDC}

From the equation (\ref{EL_fin})
\begin{equation}
        M \ddot {q}+ C(q ,\dot{q})\dot{q}+ G(q)+ f(\dot{q})+d =  \tau \nonumber
    \end{equation}
    where $q(t) \in \mathbb{R}^n;\tau \in \mathbb{R}^n;
    M(q)\in \mathbb {R}^{n \times n};
    C(q ,\dot{q}) \in \mathbb{R}^{n \times n};
    G(q) \in \mathbb{R}^n;
    f(\dot{q}) \in \mathbb{R}^n$
    
\begin{equation}\label{EL_lump}
    \bar{{M}}\ddot {{q}}+ {N}({q},\dot{{q}},\ddot{{q}}) = \tau
\end{equation}
where ${N}({q},\dot{{q}},\ddot{{q}}) \triangleq ({M}({q}) - \bar{{M}}) \ddot {{q}} + {C}{q} ,\dot{{q}})\dot{{q}}+ {G}({q})+ {f}(\dot{{q}})+{d}$ with $\bar{{M}} > 0$, a constant matrix. 

The control input for tracking a desired trajectory $q^d$ is defined as:
\begin{align}
 \tau &= \bar{{M}} {{u}}+ \hat{{N}} \label{EL_InverseDynamics}
({q},\dot{{q}},\ddot{{q}}) \\ \nonumber
{u} &= {u}_0 + \Delta {u} \\ \nonumber
{u}_0 &= \ddot{{q}}^d + {K}_D \dot{{e}} + {K}_P {e}  \nonumber
\end{align}
where $u$ is the auxiliary control input, $\hat{N}$ is the nominal value of $N$,
${e} \triangleq {q}^d - 
{q}$ is the tracking error; ${K}_P,{K}_D \in \mathbb{R}^{n \times n}$ are positive definite matrices and $\Delta {u}$ is the adaptive control term. \\

$\hat{{N}}$ is the estimated value of ${N}$ computed via past input and state data as follows with $L>0$, a small time delay to reduce the  modelling effort of a complex system as follows:
\begin{equation}\label{TDC_approx}
{\hat{{N}}({q},\dot{{q}},\ddot{{q}}) \cong {N}({q}({t} - {L}),\dot{{q}}({t} - {L}),\ddot{{q}}({t} - {L})) =  \tau({t} - {L}) - \bar{{M}}\ddot {{q}}({t} - {L})}
\end{equation}

The Time Delay Estimation(TDE) error which is the approximation error in TDC remains bounded for the system (\ref{EL_lump}) if $\bar{M}$ is selected such that
\begin{equation}\label{M_Condition}
    \lVert M^{-1}\bar{M} - I \rVert < 1
\end{equation}
The original system (\ref{EL_fin}) is delay-free. In TDC, however, the time delay $L$ in equation (\ref{TDC_approx}) is intentionally introduced to approximate the term $N$ using time-delayed input and state information. This approach reduces the modeling effort, giving TDC an advantage over traditional robust control strategies such as in \cite{ROC_5, Ref:slotine}. For instance, to design its control input, the designs in \cite{ROC_5, Ref:slotine} require nominal knowledge of both $M$ and $N$ \cite{ROC_5}. 
On the other hand, TDC only requires knowledge of the range of perturbation in the mass matrix $ M$ to design the control law, as shown in (\ref{EL_InverseDynamics})-(\ref{TDC_approx}) and (\ref{M_Condition}). Over the last two and a half decades, TDC's simplicity and effectiveness have benefited shape memory alloys \cite{TDCterminalSM}, ionic polymer metal composite actuators \cite{Ref:ipmc}, aerial vehicles \cite{Ref:new1}, wheeled mobile robots \cite{Ref:33}, underwater vehicles \cite{Ref:34}, manipulators \cite{TDCmotioncontrol, Ref:lee1}, humanoids \cite{Lee_humanoid_49}, electro-hydraulic actuators \cite{Ref:eh}, fuel-cell systems \cite{Ref:fuel}, and synchronous motors \cite{Ref:sm}. It has been demonstrated that the conventional TDC outperforms the traditional PID controller \cite{chin1994experimental,TDC_tiltRotor_43} or a class of adaptive sliding mode controllers \cite{roy2015robust}.

\section{Contribution of this Thesis}

In the view of above discussion, this thesis contributes to the following direction:
\begin{itemize}
    \item A time delay-based adaptive controller is formulated can tackle state-dependent uncertainty and unmodelled dynamics without their a priori knowledge, regardless of the nature of these as seen in the applications to payload delivery with both bipedal walking and quadrotors.
    
    \item The formulated controller is simpler as it is capable of controlling a bipedal in all the phases of walking of single support stage, double support stage.(cf. \cite{mu2004development, mitobe1997control, braun2009control}) and accounts for constraint forces without the need for additional computation (unlike cf. \cite{li2013adaptive, sun2016adaptive}).  
    
    \item  The proposed method for a quadrotor
system, is a first of its kind, because the existing TDE methods for quadrotors  (\cite{lee2012experimental, wang2016model, dhadekar2021robust}) are non-adaptive solutions, to the best of the author's knowledge.

\item The closed-loop stability of the proposed controller has been verified analytically. The performance of the formulated control design has been verified via realistic experiments with a high fidelity simulator along with hardware experiments(in the case of Quadrotor). To validate the efficacy of the controller design, the obtained results are compared against state-of-the-art methods. 
\end{itemize}

\section{Thesis Organisation}

This thesis is organised into 4 chapters as follows:

\begin{itemize}
    \item \textbf{Chapter 1:} This is an introductory chapter detailing the motivation behind the research, the contributions of the thesis and its outline.
    \item \textbf{Chapter 2:} This chapter introduces Time Delay based Controller(TDC) for all the stages of bipedal walking together. The Euler-Lagrange dynamics for the bipedal has been introduced. The uncertainties are lumped and is estimated using time delay control strategies. The closed loop stability is established analytically using Lyapunov theory. Multiple simulation scenarios are provided to establish the performance improvement over the state-of-the-art controllers.
    \item \textbf{Chapter 3:} This chapter first introduces quadrotor co-design dynamics wherein the position and attitude dynamics are partly decoupled into outer loop and inner loop respectively. Time delay based controller is designed to address the unmodelled, uncertain parameters that are involved while carrying a payload of unknown mass via a quadrotor. The closed loop stability is established analytically and experimental scenario performed on hardware is discussed.
    
    \item \textbf{Chapter 4:} This chapter concludes the thesis and a briefly discusses the possible future work.
\end{itemize} 

%% file: 03_Bipedal.tex
\chapter{Adaptive Artificial Time Delay Control for Bipedal Walking with Robustification to State-dependent Constraint Forces}
\label{bipedal_paper}

\section{Introduction}
Artificial time delay based control is a control strategy requiring limited system knowledge: it was proposed in, by approximating uncertainties via input and state information of immediate past instant. Owing to its simplicity in implementation and significantly low computation burden, adaptive control using artificial delay found remarkable acceptance in the robotics community in the last decade, including bipedal robot control.

Therefore, based on the proven benefits of artificial time delay based designs over other adaptive schemes, an obvious question arises: can the state-of-the-art biped controllers,
be extended to handle the unmodelled constraint forces? Unfortunately, we do not get a positive answer for this question as constraint forces are state-dependent terms and cannot be bounded a priori; whereas, cannot handle state-dependent unknown uncertainties (and discussion later in Remark 4). In fact, under such uncertainty setting, instability cannot be ruled out for adaptive controllers which are built on the assumption of a priori bounded uncertainty.
In view of the above discussion, an adaptive-robust TDE (ARTDE) scheme is designed for bipedal walking with the following contributions:

\begin{itemize}
    
    \item The proposed controller is simpler as it does not require different controllers for different phases of walking motion (cf. \cite{mu2004development, mitobe1997control, braun2009control}) or does not require to compute constraint forces separately (cf. \cite{li2013adaptive, sun2016adaptive}). 
    \item Unlike \cite{lee2016robust, Ref:hum, pi2019adaptive}, the proposed adaptive TDE method is also \emph{robust} (hence called adaptive-robust TDE) to constraint forces, which are considered as state-dependent unmodelled dynamics. 
    
\end{itemize}

The rest of the chapter is organised as follows: Section 2.2 describes the Euler Lagrange Dynamics;
Section 2.3 details the proposed control framework, while corresponding stability analysis is provided
in Section 2.4; comparative simulation results are provided in Section 2.5.

\section{System Dynamics and Problem Formulation}\label{sec TDE}

Let us consider the following class of $n$ degrees-of-freedom (DOFs) dynamics for the biped robot \cite{park2001impedance, pi2019adaptive} 
\begin{equation}\label{sys}
{ M(   q)\ddot{  q}+  H( q,\dot{ q})= \tau},
\end{equation}
where ${q}\in\mathbb{R}^{n}$ is the joint position; $\tau\in\mathbb{R}^{n}$ is the control input; ${M(q)}\in\mathbb{R}^{n\times n}$ is the mass/inertia matrix and ${ H( q,\dot{ q})}\in\mathbb{R}^{n}$ combines other system dynamics terms (e.g., Coriolis, friction, gravity) including unmodelled dynamics and disturbances. The following standard property holds 

\textbf{Property} (\cite{park2001impedance,lee2016robust, Ref:hum, pi2019adaptive}):
The matrix ${M}(  q)$ is uniformly positive definite for all $ q$, i.e., $\exists \psi_1, \psi_2 \in \mathbb{R}^{+}$ such that 
\begin{equation}
\psi_1 {I} \leq {M}(  q) \leq \psi_2 {I} \Rightarrow (1/\psi_2) {I} \leq {M}^{-1}(  q) \leq (1/\psi_1) {I}. \label{prop}
\end{equation}
\begin{remark}[On the system dynamics]
As an alternative to the unconstrained system dynamics \eqref{sys}, some researchers have proposed constrained multi-modal dynamics (cf. \cite{mu2004development, mitobe1997control, braun2009control, li2013adaptive, sun2016adaptive}). However, the latter approach involves different dynamics for different walking phases (e.g., single support, double support, impulse etc.), requiring different controller for each phases, making the control design and analysis comparatively difficult \cite{braun2009control}. Dynamics \eqref{sys} can simplify the control design provided it can handle the ground reaction forces acting as impulsive unmodelled dynamics.
\end{remark}

\begin{assumption}[Uncertainty setting]\label{assm_1}
Inertia matrix $ M$ is not precisely known, but its upper bound $\psi_2$ from (\ref{prop}) is known. At the same time, the state-dependent (via ${q}$ and $\dot{ q}$) dynamics term ${H}$ in \eqref{sys} is unknown. Hence, ${H}$ cannot be considered to be bounded by a constant a priori to the control design \cite{roy2020adaptive, roy2019overcoming, roy2019role}. 
\end{assumption}
\begin{remark}[Generality of the proposed approach]
Bipedal motion from a constrained and an unconstrained dynamics are related via state-dependent dynamics terms \cite{braun2009control}. Therefore, Assumption \ref{assm_1} allows to include constraint forces in the unified dynamics \eqref{sys} as unmodelled dynamics. Nevertheless, state-of-the-art adaptive designs cannot handle such a challenge for biped walking problem (cf. Remark 4).
\end{remark}


\begin{assumption}
The desired trajectory ${q}^d(t)$ is designed such that ${q}^d, \dot{ q}^d, \ddot{ q}^d \in \mathcal{L}_{\infty}$. Further, to avoid kinematic singularity, the desired knee angle trajectory is designed such that the knee is never fully-stretched (cf. \cite{park2001impedance}).
\end{assumption}

In view of the aforementioned discussion, the control problem is defined as below:

\textbf{Control Problem:} Under Assumptions 1-2 and Property \eqref{prop}, design an adaptive control for biped walking motion while negotiating unknown state-dependent uncertainty stemming from constraint forces (in line with Remark 2).

\section{Controller Design and Analysis}\label{sec TDE}
Before presenting the proposed controller, the biped dynamics \eqref{sys} is re-arranged by introducing a constant and positive definite diagonal matrix ${\bar { M}}$ as
\begin{align} 
&{\bar{ M}\ddot { q}} + {{N}}({{q}},{{\dot { q}},\ddot{ q}}) = { \tau },\label{robotdynamics2}\\
\text{with}~~~~&{{N}}({{q}},{{\dot { q}},\ddot{ q}}) = [{{M}}({{q}}) - {\bar { M}}]{\ddot { q}} +  H( q,\dot{ q})\label{nonlinearterms}
\end{align}
and the choice of ${\bar { M}}$ is discussed later (cf. discussion after \eqref{beta}). Note that owing to Assumption 1, the unknown dynamics is now subsumed under $ N$.

Let us define the tracking error as ${e}(t)={q}^d(t)-{q}(t)$. Subsequently, variable dependency will be omitted whenever obvious for brevity. The control input $\tau$ is designed as
\begin{subequations}\label{tau}
\begin{align}
\tau &={\bar{ M}  u+\hat{ N}}({{q}},{{\dot { q}},\ddot{ q}}), \label{input}\\
{u} &= {u}_0+\Delta  u,\label{tarc input}\\
 u_0 &=\ddot{ q}^d +  K_D\dot{ e} +  K_P  e, \label{aux}
\end{align}
\end{subequations}
where $ K_P,  K_D \in \mathbb{R}^{n \times n}$ are two positive definite matrices; $\Delta  u$ is the adaptive control term to be designed later and ${\hat{ N}}$ is the estimated value of ${N}$ computed via the past input and state data as \cite{Ref:6-2,Hsia1991} 
\begin{equation}\label{approx}
{\hat{ N}(  q,\dot{ q} , \ddot{ q})\cong}  N( q_L,\dot{ q}_L,\ddot{ q}_L)=\tau_L-\bar{ M}\ddot{ q}_L,
\end{equation}
where $L>0$ is a small time delay. 
\begin{remark}[Artificial time delay]
The uncertainty estimating process via past data (i.e., time delayed data) as in \eqref{approx} is typically called in literature as \textit{time-delay estimation} (TDE) or \textit{artificial delay} based estimation method, since time delay is invoked into the system artificially/intentionally via past data, while the original system was free of any time delay. Since the TDE process \eqref{approx} relies on immediate past data, $L$ is typically selected in practice as the sampling interval of available hardware \cite{Hsia1991,Ref:6-2, Ref:9_1, Ref:12,Ref:jin,Ref:jin2017model,Ref:roy2017adaptive}.
\end{remark}

Substituting (\ref{input}) in (\ref{robotdynamics2}), one obtains
\begin{align}
\ddot{ e}& =- K_D\dot{ e}_{L}- K_P  e_{L}+ \sigma, \label{error dyn delayed 2}
\end{align}
where $\sigma={\bar{ M}}^{-1}({ N}-{\hat{ N}})$ represents the \textit{estimation error stemming from (\ref{approx})}, also termed as \textit{TDE error}.  

Design of the adaptive control term $\Delta  u$ relies on the upper bound structure of $\sigma$ as derived subsequently, followed by the proposed adaptive law.

\subsection{Upper bound structure of $ \sigma$} \label{sec bound}
From (\ref{nonlinearterms}) and (\ref{error dyn delayed 2}), the following relations can be achieved:
\begin{align}
\hat{ N}&={N}_L=[{{M}}({{q}_L}) - {\bar { M}}]{\ddot { q}_L} +  H_L ,\label{sig 2} \\
 \sigma&=\ddot{ q}- u. \label{sig 1} 
\end{align}
Using (\ref{sig 2}), the control input $ \tau$ in  (\ref{input}) can be rewritten as
\begin{align}
 \tau &= \bar{ M}  u+[{ {M}}({{q}_L}) - {\bar { M}}]{\ddot { q}_L} +  H_L. \label{tau new}
\end{align}
Multiplying both sides of (\ref{sig 1}) with $ M$ and using (\ref{sys}) and (\ref{tau new}) we have
\begin{align}
{M}  \sigma  &=  \tau  - H-{M} {u}, \nonumber \\
& = \bar{ M}  u+[{ {M}}({{q}_L}) - {\bar { M}}]{\ddot { q}_L} +  H_L - H -{M} {u}. \label{sig 3}
\end{align}
Defining ${K} \triangleq [{K}_P ~ {K}_D]$ and using (\ref{error dyn delayed 2}) we have
\begin{align}
\ddot{{q}}_L&=\ddot{{q}}^d_L-\ddot{{e}}_L =\ddot{{q}}^d_L + K \xi_L- \sigma_L+\Delta  u_L.\label{sig 4}
\end{align}
Substituting (\ref{sig 4}) into (\ref{sig 3}), and after re-arrangement yields
\begin{align}
 \sigma &= \underbrace{{M}^{-1}\bar{{M}}( \Delta  u-  \Delta  u_L)}_{ \chi_1}+\underbrace{{M}^{-1}({M}_L \Delta  u_L-{M} \Delta  u)}_{ \chi_2} \nonumber
\end{align}
\begin{align}
&+\underbrace{{M}^{-1}\lbrace \bar{{M}}\ddot{{q}}^d-({M}-{M}_L+\bar{{M}})\ddot{{q}}^d_L+ H_L - H \rbrace }_{ \chi_3} \nonumber\\
& +\underbrace{{M}^{-1}({M}_L-\bar{{M}}){K} \xi_L}_{ \chi_4}-\underbrace{{M}^{-1}({M}_L-\bar{{M}}) \sigma_L}_{ \chi_5} \nonumber\\
&+\underbrace{({M}^{-1}\bar{{M}}- I){K} \xi}_{ \chi_6}. \label{sig 5}
\end{align}
The following Lemma provides the upper bound of $||  \sigma ||$:
\begin{lem}[\cite{roy2019new, roy2020new}]
Under the condition
\begin{equation}
   ||  E ||= ||  {I} - {M}^{-1}( q)\bar{ M} || <1, \label{E} 
\end{equation}
and the property (\ref{prop}), there exist (unknown) scalars $\delta_{1,2,\cdots,5}$, such that
\begin{align}
|| \chi_{1,2,3} || &\leq \delta_{1,2,3},~|| \chi_{4} || \leq || E   K ||  || \xi || +\delta_4,\\
|| \chi_{5} || &\leq ||{ E}|| || \sigma ||  + \delta_5,~ || \chi_{6} || \leq || {E}{K}|| || \xi|| 
\end{align}
yielding
\begin{align}
 \lVert  \sigma \rVert &\leq \beta_0 + \beta_1 \lVert \xi \rVert, \label{sig bound} \\
 \text{where}~ &\beta_0=\frac{\sum_{i=1}^{5}\delta_i}{1-\lVert  E \rVert}, ~ \beta_1=\frac{2 \lVert  E   K \rVert  }{1-\lVert  E\rVert} \label{beta}.
\end{align}
\end{lem}

The condition \eqref{E}, which is standard in the literature of TDE based controllers \cite{Hsia1991,Ref:6-2, Ref:9_1, Ref:12,Ref:jin,Ref:jin2017model,Ref:roy2017adaptive, lee2016robust, Ref:hum, pi2019adaptive}, gives the criterion to select $\bar{ M}$, which is feasible since upper bound knowledge of ${M}$ is available from Assumption 1. 
\begin{remark}[State-dependent TDE error bound]\label{rem_sig}
The upper bound structure of TDE error $ \sigma$ in \eqref{sig bound} has state-dependency via $\beta_1 ||  \xi ||$, implying $ \sigma$ cannot be considered bounded a priori. \textit{We will show later that assuming a priori boundedness not only sacrifices tracking accuracy, but may cause instability during bipedal walking (cf. Sect. IV.B)}.
\end{remark}

\subsection{Design of the Adaptive Control Law $\Delta  u$} \label{sec gain}
The term $\Delta  u$ is designed as
\begin{align}
&\Delta {{u}} =\alpha c ~\text{sig}({s},\epsilon),  \label{delta u2}
\end{align}
where ${s=B}^T{P}\xi$, $\xi= \begin{bmatrix}
{ e}^T&
\dot{ e}^T
\end{bmatrix}^T$ and ${P>0}$ is the solution of the Lyapunov equation ${A}^T{P+PA=-Q}$ for some ${Q>0}$, where $  A=\begin{bmatrix}
{0} & {I} \\
-{K}_P & -{K}_D
\end{bmatrix}$, ${B}=\begin{bmatrix}
{0}\\
{I}
\end{bmatrix}$; $\alpha \in \mathbb{R}^{+}$ is a user-defined scalar; ${c} \in \mathbb{R}^{+}$ is the overall switching gain tackling $ \sigma$ and $\text{sig}({s},\epsilon) \triangleq {{s}}/{\sqrt{|| {s}||^2+\epsilon}}.$
Here, $\epsilon$ is a small positive scalar used to avoid chattering. 
 


The switching gain $c$ in (\ref{delta u2}) is formulated based on the structure of $|| \sigma ||$ as
\begin{align}
c=\hat{\beta}_0+\hat{\beta}_1 || \xi || , \label{sw gain} 
\end{align}
where $\hat{\beta}_0, \hat{\beta}_1$ are the estimates of $\beta_0, \beta_1 \in \mathbb{R}^{+}$, respectively. 
The gains are evaluated as follows: 
\begin{align}
\dot{\hat{\beta}}_j =
  \begin{cases}
    \gamma_j \lVert \xi \rVert^j \lVert  s \rVert,       &  {\text{if} ~ \text{any}~\hat{\beta}_j \leq \underline{\beta}_j ~ \text{or} ~{s}^T \dot{{s}} >0} \\
    - \gamma_j \lVert \xi \rVert^j \lVert  s \rVert,       & {\text{if} ~ {s}^T \dot{{s}}  \leq 0 ~\text{and all}~  \hat{\beta}_j > \underline{\beta}_j }
  \end{cases}, \label{ATRC} 
\end{align}
\noindent with $\hat{\beta}_j(0) \geq \underline{\beta}_j >0$, $j=0,1$ are user-defined scalars. 

Combining (\ref{input}), (\ref{tarc input}), (\ref{approx}) and (\ref{delta u2}), ARTDE becomes
\begin{equation}\label{theproposedcontrollaw}
\begin{split}
{{\tau }} &= \underbrace {{{{\tau }}_L} - {\bar{ M}}{{{\ddot { q}}}_L}}_{{\text{TDE~part}}}+\underbrace {{\bar{ M}}({{{\ddot { q}}}^d} + {{{K}}_D}{\dot { e}} + {{{K}}_P}{{e}})}_{{\text{Desired~dynamics~injection~part}}}+\underbrace {\alpha {\bar{ M}} c \text{sig}({{s}},\epsilon).}_{{\text{Adaptive-robust~control~part}}}
\end{split}
\end{equation}
\begin{remark}
In \eqref{sig bound}, the term $\beta_0$ can capture the effects of bounded external impulsive forces, while $\beta_1 ||  \xi ||$ can capture state-dependent unmodelled dynamics, particularly, constraint forces (along with uncertain system parameters). Therefore, compared to \cite{pi2019adaptive}, estimating these parameters via the adaptive laws \eqref{ATRC} helps the proposed design avoid design complexities computing constraint forces separately. 
\end{remark}

\begin{algorithm}[!h]

 \caption{ {Design steps of the proposed controller}}
 \vspace{0.2cm}
 
\textbf{Step 1 (Defining the error variables):} Find $e$ via (\ref{tau}); define gains ${K_P}, {K_D}$ and solve for $P$ from (\ref{delta u2}) to compute $s = B^T P\xi$. 

\textbf{Step 2 (Designing adaptive gains):} Using variables from Step 1, compute gain $c$ from (\ref{sw gain}) using the adaptive laws (\ref{ATRC}). 

\textbf{Step 3 (Computing $\tau$):} Select $\overline{M}$ from (\ref{E}); then, using results from Steps 1-2, compute control input $\tau$ from (\ref{theproposedcontrollaw}) after deriving $\hat{N}$ via TDE method (\ref{approx}).

\textbf{Step 4 (Control input to system):} Finally, apply  $\tau$ from Steps 3 to the bipedal for walking.
\vspace{0.2cm}
 \end{algorithm}

\begin{figure}[!t]
		\centering
		\includegraphics[width=1\linewidth]{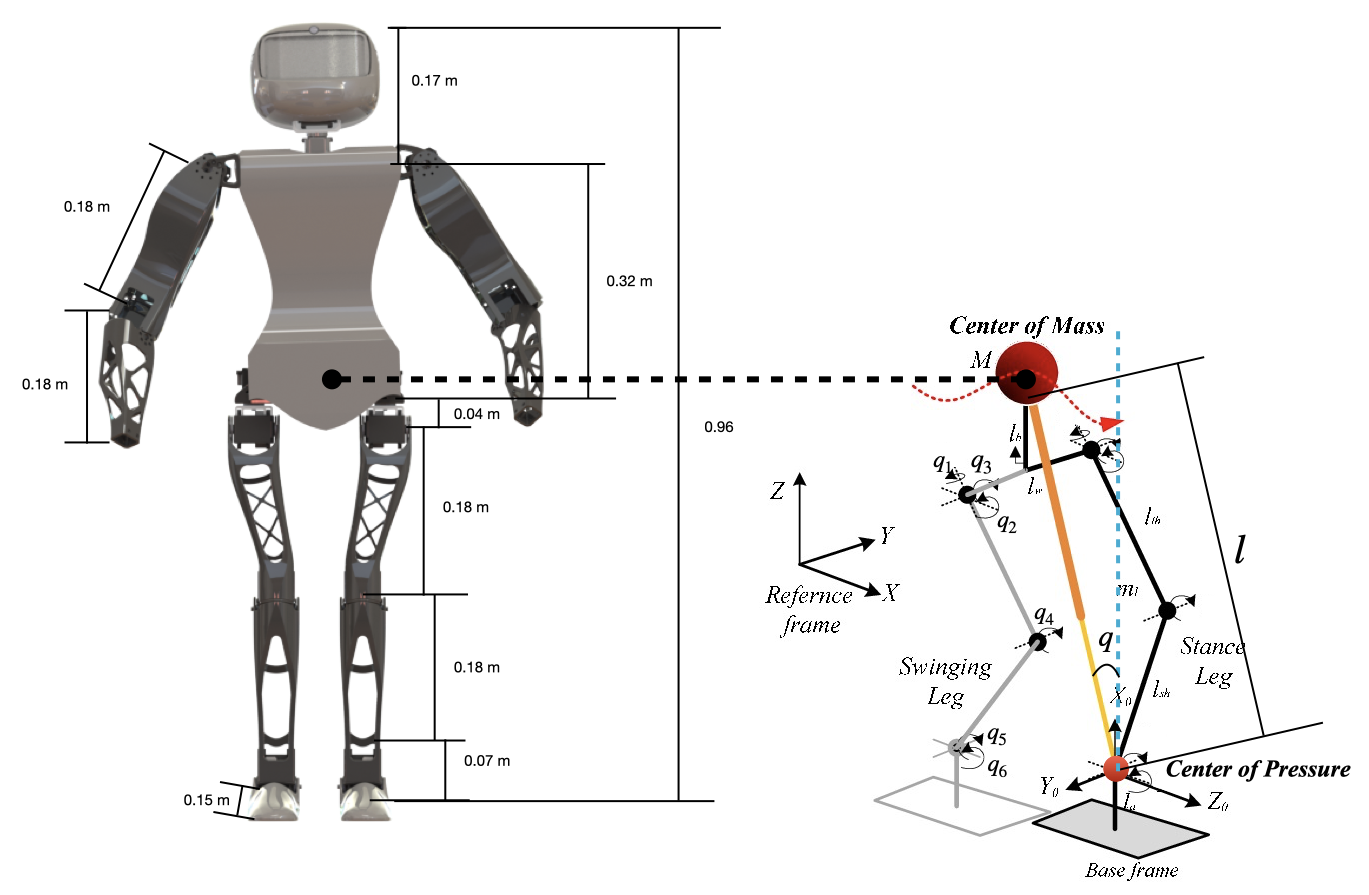}
		\caption{Schematic of Ojas (proposed humanoid).} \label{fig:humanoid}
\end{figure}

\section{Closed-Loop System Stability}\label{sec stability}
The stability analysis of TDARC is carried out utilizing the following Lyapunov function candidate:
\begin{align}
\bar{V}&=V+\frac{(\hat{\beta}_0-\beta_0^{*})^2}{({2 \gamma_0})}+\frac{(\hat{\beta}_1-\beta_1^{*})^2}{({2 \gamma_1})}, \label{lyapunov1} 
\end{align}
where $V(\xi)=\frac{1}{2} \xi^T{P}  \xi$ and $\beta_j^{*} \geq \beta_j(t)>0$ is a constant. For the ease of analysis, we define a region such that
\begin{align}
&\alpha \frac{||{s}||^2}{\sqrt{|| {s}||^2+\epsilon}} \geq ||  s|| 
\Rightarrow  || s|| \geq \sqrt{\frac{\epsilon}{\alpha^2-1}} \triangleq \varphi. \label{region}
\end{align}
The condition (\ref{region}) implies that one needs to select $\alpha>1$, which is always possible since $\alpha$ is a user-defined scalar. The closed-loop system stability is stated in the following theorem:
\begin{theorem}
The system (\ref{sys}) employing TDARC with the controller (\ref{theproposedcontrollaw}), (\ref{ATRC}) is Uniformly Ultimately Bounded (UUB). 
\end{theorem}
\begin{proof}
Exploring the various combinations of $ \Delta {u}$, the gains $\hat{\beta}_j, j=0,1$ in (\ref{delta u2}), (\ref{ATRC}) and the condition (\ref{region}), the stability of the overall system is analyzed for the following four possible cases using the common Lypaunov function (\ref{lyapunov1}): 

\textbf{Case (i):} $|| {s} || \geq \varphi ~\text{and}  \lbrace \text{any}~\hat{\beta}_j \leq \underline{\beta}_j ~ \text{or} ~{s}^T \dot{{s}}>0 \rbrace$\\
Using the Lyapunov equation ${A}^T{P+PA=-Q}$, the time derivative of (\ref{lyapunov1}) yields
\begin{align}
\dot{\bar{V}} & \leq -(1/2)\xi^{T}{Q} \xi+  {s}^{T} \lbrace-\alpha c ({{s}}/{\sqrt{|| {s}||^2+\epsilon}})+ \sigma \rbrace +({(\hat{\beta}_0-\beta_0^{*})/{\gamma_0})\dot{\hat{\beta}}_0} +({(\hat{\beta}_1-\beta_1^{*})/{\gamma_1})\dot{\hat{\beta}}_1} \nonumber\\ 
& \leq - (1/2)\xi^{T}{Q} \xi- c || {s} || + (\beta_0^{*}+\beta_1^{*}||  \xi ||) || {s}|| + (\hat{\beta}_0-\beta_0^{*})  ||  s ||+(\hat{\beta}_1-\beta_1^{*}) || \xi|| ||  s ||  \nonumber\\ 
& \leq -(1/2) \lambda_{\min}( Q)||{ \xi}||^2 \leq 0,   \label{case 1}
\end{align}

\noindent as $\alpha > 1$. From (\ref{case 1}) it can be inferred that ${\bar{V}}(t) \in \mathcal{L}_{\infty}$ implying $ \xi(t), \hat{\beta}_j(t) \in \mathcal{L}_{\infty} \Rightarrow  \sigma(t),  \Delta {u} \in \mathcal{L}_{\infty}$ for Case (i). 

 \textbf{Case (ii):} $|| {s} || \geq \varphi ~\text{and}~  \lbrace {s}^T \dot{{s}}  \leq 0 ~\text{and all}~  \hat{\beta}_j > \underline{\beta}_j \rbrace $ \\
For this case, the time derivative of (\ref{lyapunov1}) yields
\begin{align}
\dot{\bar{V}} &\leq - (1/2)\xi^{T}{Q} \xi- c || {s} || + (\beta_0^{*}+\beta_1^{*}||  \xi ||) || {s}|| - (\hat{\beta}_0-\beta_0^{*})  ||  s || - (\hat{\beta}_1-\beta_1^{*}) || \xi|| ||  s || \nonumber  \\
& \leq -(1/2) \lambda_{\min}( Q)||{ \xi}||^2 +2 (\beta_0^{*}+\beta_1^{*}||  \xi ||) ||  s ||  .\label{case 2}
\end{align}
In this case we have ${s}^T \dot{{s}} \leq 0$ which implies $|| {s} ||, ||  \xi ||  \in \mathcal{L}_\infty$ (cf. the relation ${s} =  B^T  P  \xi $). Thus, $\exists \varsigma \in \mathbb{R}^{+}$ such that $2 (\beta_0^{*}+\beta_1^{*}||  \xi ||) ||  s || \leq \varsigma $. Further, considering a scalar $z$ as $0 < z < (1/2) \lambda_{\min}( Q) $ one has 
\begin{align}
\dot{\bar{V}} & \leq -\lbrace (1/2) \lambda_{\min}( Q)-z \rbrace||{ \xi}||^2- z||{ \xi}||^2 + \varsigma  .\label{case 2_1}
\end{align}
The gains $\hat{\beta}_1, \hat{\beta}_2 \in \mathcal{L}_{\infty}$ in Case (i) and decrease in Case (ii). This implies $\exists \varpi \in \mathbb{R}^{+}$ such that $\sum_{j=0}^{1}{(\hat{\beta}_j-\beta_j^{*})^2}/{ \gamma_j} \leq \varpi$. Therefore, the definition of $\bar{V}$ in (\ref{lyapunov1}) yields
\begin{equation}
\bar{V} \leq   \lambda_{\max} ({P}) ||{ \xi}||^2 + \varpi. \label{V}
\end{equation}
Using the relation (\ref{V}), (\ref{case 2_1}) can be written as
\begin{equation}
\dot{\bar{V}}  \leq - \upsilon \bar{V} - z||{ \xi}||^2 + \varsigma+ \upsilon \varpi, \label{V_dot}
\end{equation}
where $\upsilon \triangleq ({\frac{1}{2} \lambda_{\min}( Q)-z})/{\lambda_{\max} ({P})} $. Hence, $\dot{\bar{V}} <0$  would be achieved when $||{ \xi}|| \geq \sqrt{({\varsigma+ \upsilon \varpi})/{z}} $. 

\par \textbf{Case (iii):} $ || {s} || < \varphi ~\text{and} \lbrace \text{any}~\hat{\beta}_j \leq \underline{\beta}_j ~ \text{or} ~{s}^T \dot{{s}}>0 \rbrace $  \\
The fact $|| {s} || < \epsilon$ implies that $\exists \bar{\epsilon} \in \mathbb{R}^{+}$ such that $||  \xi || \leq \bar{\epsilon} $ from the relation ${s}={B}^T {P}  \xi$.  Using (\ref{delta u2}) we have
\begin{align}
\dot{\bar{V}} &\leq -({1}/{2})\xi^{T}{Q} \xi+ {s}^{T} \lbrace-\alpha c ({{s}}/{\sqrt{|| {s}||^2+\epsilon}})+ \sigma \rbrace +({(\hat{\beta}_0-\beta_0^{*})/{\gamma_0})\dot{\hat{\beta}}_0} +({(\hat{\beta}_1-\beta_1^{*})/{\gamma_1})\dot{\hat{\beta}}_1} \nonumber\\
& \leq -(1/2) \lambda_{\min}({Q})|| \xi||^2+ (\hat{\beta}_0+\hat{\beta}_1||  \xi ||) ||  s ||. \label{case 3}
\end{align}
Unlike Case (i), proving boundedness of $\hat{\beta}_j$ in Case (iii) demands that $\hat{\beta}_j$s start decreasing in a finite time, i.e., ${s}^T \dot{{s}} \leq 0$ should occur (from the second law of (\ref{ATRC})) in a finite time. For this, we need to investigate only the evaluation of $V$, where gains only increase implying $\hat{\beta}_j > \underline{\beta}_j$. The condition ${s}^T \dot{{s}} >0$ in Case (iii) implies $|| {s} ||$ is increasing; thus $\exists \delta \in \mathbb{R}^{+}$ such that $|| {s} || \geq \delta$. Further, using $|| {s} || \leq || {B}^T {P}|| ||  \xi ||$ we have
\begin{equation}
\delta \leq || {s} ||  \leq || {B}^T {P}|| ||  \xi || \Rightarrow ||  \xi || \geq \delta / || {B}^T {P}|| . \label{lem p1}
\end{equation}
Then, using (\ref{lem p1}), the adaptive law (\ref{ATRC}) yields 
\begin{equation}
\dot{\hat{\beta}}_0 \geq \gamma_0 \delta, ~\dot{\hat{\beta}}_1 \geq (\gamma_1 \delta^2 ) / || {B}^T {P}||. \label{lem p2}
\end{equation}
Using (\ref{ATRC}) and the fact $||  s || < \varphi$ for Case (iii), the time derivative of $V(\xi)=(1/2) \xi^T{P}  \xi$ for Case (iii) yields
\begin{align}
\dot{{V}} &\leq -({1}/{2}) \lambda_{\min}( Q)|| \xi||^2+ {s}^{T} \lbrace -\alpha c ({{s}}/{\sqrt{|| {s}||^2+\epsilon}})+ \sigma \rbrace  \nonumber\\
& \leq -({ \lambda_{\min}( Q)}/{ \lambda_{\max}( P)})V + ({\beta}_0^{*}+{\beta}_1^{*} ||  \xi ||)|| {s} || - \alpha ( \hat{\beta}_0 + \hat{\beta}_1 \lVert  \xi \rVert ) ({ \delta || {s} ||}/{\sqrt{\varphi^2+\epsilon}}) . \label{lem p3}
\end{align}
If $||  \xi ||$ decreases, then it would also ensure that $||{s}||$ decreases (i.e., ${s}^T \dot{{s}} <0$) as ${s}={B}^T {P}  \xi$. Consequently, $\hat{\beta}_0, \hat{\beta}_1$ start decreasing following (\ref{ATRC}) and hence they would remain bounded individually. This feature can be realized if $\dot{V} < -\frac{ \lambda_{\min}( Q)}{ \lambda_{\max}( P)}V$ is established.
Such condition can be achieved from (\ref{lem p3}) when 
\begin{align}
\alpha \hat{\beta}_0(\delta/ \varrho) \geq \beta_0^{*}, ~~\alpha \hat{\beta}_1(\delta/ \varrho) \geq \beta_1^{*}, \label{lem p4}
\end{align}
where $\varrho \triangleq \sqrt{\varphi^2+\epsilon}$. Since (\ref{lem p2}) defines the minimum rates of increments, (\ref{lem p4}) is satisfied within finite times $T_1,T_2$ where 
\begin{align}
T_1 \leq ({\varrho \beta_0^{*}})/({\alpha \gamma_0 \delta^2}),~~T_2 \leq ({\varrho \beta_1^{*} || {B}^T {P}||})/({\alpha \gamma_1 \delta^3}). \label{lem p5}
\end{align}
Therefore, the exponential decrease of $||  \xi ||$ and subsequent boundedness of $\hat{\beta}_j $ and $ \hat{\beta}_1$ is achieved within a finite time $T=\max\lbrace T_1~T_2 \rbrace$.
In addition,  $||  s || < \varphi$ in Case (iii) implies $||  \xi || \in \mathcal{L}_\infty$ and consequently $(\hat{\beta}_0+\hat{\beta}_1||  \xi ||) ||  s || \leq \varpi_1$, where $\varpi_1 \in \mathbb{R}^{+}$. Using these results and the procedure in (\ref{V_dot}), the relation (\ref{case 3}) can be written as
\begin{align}
\dot{\bar{V}} & \leq -\upsilon \bar{V}-  z||{ \xi}||^2 +  \varpi_1. \label{case 3_1}
\end{align}
Hence, $\dot{\bar{V}} <0$ would be established when $|| \xi|| \geq \sqrt{{\varpi_1}/{z}}$.

\textbf{Case (iv):} $|| {s} || < \varphi ~\text{and}~  \lbrace {s}^T \dot{{s}}  \leq 0 ~\text{and all}~  \hat{\beta}_j > \underline{\beta}_j \rbrace $ \\
\noindent Similarly, for this case
\begin{align}
\dot{\bar{V}} &\leq -({1}/{2})\xi^{T}{Q}  \xi+ {s}^{T} \lbrace -\alpha c ({{s}}/{\sqrt{|| {s}||^2+\epsilon}})+ \sigma \rbrace - ({(\hat{\beta}_0-\beta_0^{*})/{\gamma_0})\dot{\hat{\beta}}_0} - ({(\hat{\beta}_1-\beta_1^{*})/{\gamma_1})\dot{\hat{\beta}}_1} \nonumber\\
& \leq  -(1/2) \lambda_{\min}({Q})|| \xi||^2 +2 (\beta_0^{*}+\beta_1^{*}||  \xi ||) ||  s ||.\label{case 4}
\end{align}
This case can be analyzed exactly like Case (ii).

The stability results from Cases (i)-(iv) reveal that the closed-loop system is UUB. 
\end{proof}




\section{Verification of the Proposed ARTDE}

\subsection{Simulation Setup}
To verify the performance of the proposed controller, a $20$ DoFs humanoid named Ojas (cf. Fig. \ref{fig:humanoid} for detailed mechanical structure), weighing $8.986$ kg and $0.94$ m tall from feet to head, is designed: each leg and arm of the robot has $6$ and $4$ DoFs, respectively. We relied on the high fidelity simulator Pybullet for verification of the proposed controller. 

The objective is Ojas should maintain a desired walking motion in the face of various uncertainties. For this purpose, only the 6 joints of each leg, 3 hip joints, 1 knee joint and 2 angle joints, while other joints of the robot are kept fixed (i.e., fixed upper-body). Thus, twelve joints are operated simultaneously for both legs. Subsequently we denote $q_1=$ yaw hip joint, $q_2=$ roll hip joint, $q_3=$ pitch hip joint, $q_4=$ knee joint, $q_5=$ pitch ankle joint and $q_6=$ roll ankle joint. Following \cite{park2001impedance}, the desired walking motion is generated (cf. Fig. \ref{fig:footTraj}) via the desired trajectories for the six joints for each leg as in Fig. \ref{fig:jointTraj}, leading to a walking speed $0.2$ m/s with $1$ s step period and stride of $0.1$ m.
\begin{figure}[!t]
		\centering
		\includegraphics[width=\linewidth]{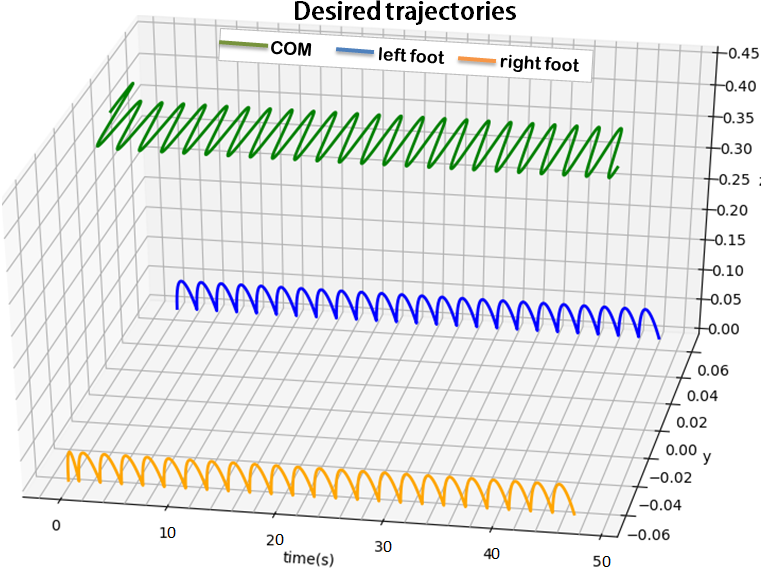}
		\caption{Desired trajectories of Center of mass (COM) and feet.}
		\label{fig:footTraj}
\end{figure}
\begin{figure}[!t]
		\centering
		\includegraphics[width=\linewidth]{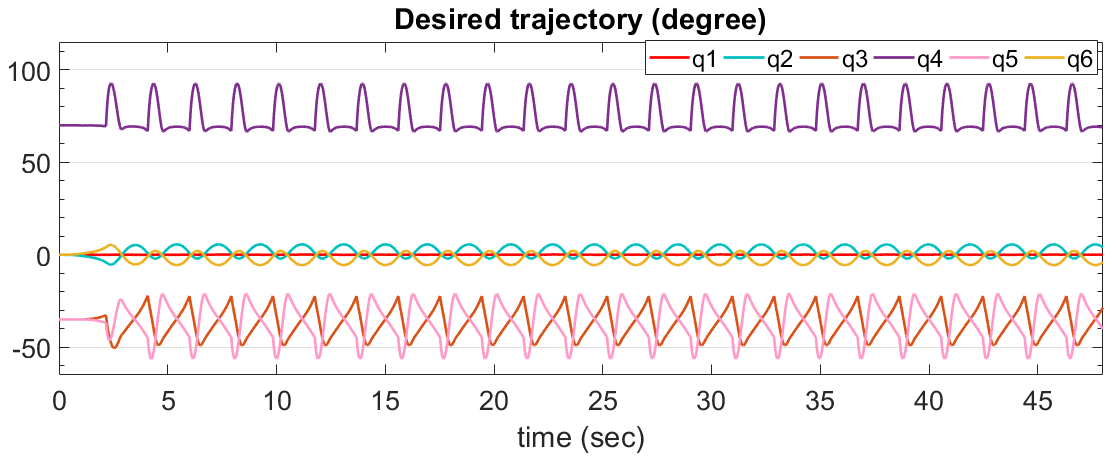}
		\caption{Desired leg joint trajectories (same for both the legs).}
		\label{fig:jointTraj}
\end{figure}

\begin{figure}[!t]
		\centering
		\includegraphics[width=.9\linewidth]{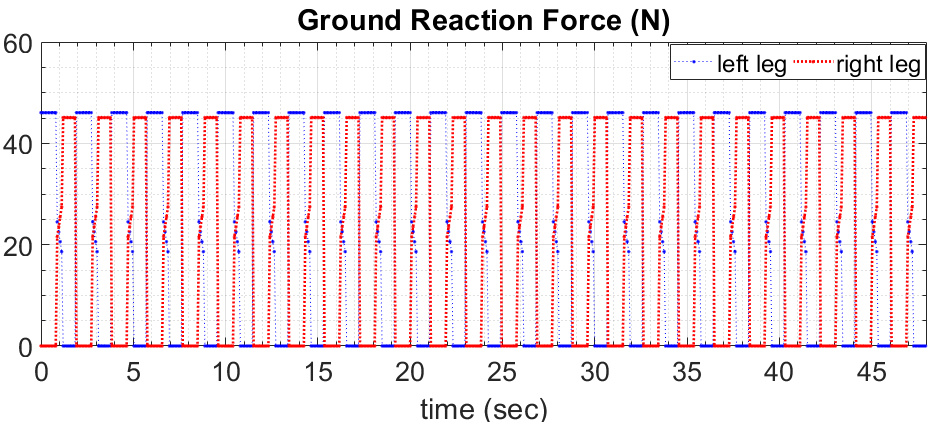}
		\caption{Ground reaction force on the legs.}
		\label{fig:grf}
\end{figure}
\subsection{Simulation Scenario, Results and Analysis}
{To properly judge the effectiveness of the proposed ARTDE scheme against state-of-the-art, we consider the classical TDE \cite{Hsia1991} (control law \eqref{tau} with $\Delta {u =0} $) and the adaptive TDE biped controller \cite{pi2019adaptive} (called ATDE henceforth), with $\Delta {u}$ as (\ref{delta u2}) and the following adaptive law for $c$  
\begin{align}
\dot{c}=\begin{cases}
    \gamma_0 ||  s ||,       &  \text{if} ~ c \leq \underline{\beta}_0 ~ \text{or} ~(||{s}|| - || {s}_L||) >0 \\
    -  \gamma_0 ||  s ||,      & \text{if} ~ (||{s}|| - || {s}_L||) \leq 0
  \end{cases}. \label{gain_old}
\end{align}}
Three simulation scenarios, S1, S2 and S3, are considered in following subsections with same control design parameters as: $\bar{ M}=0.042  I$ (kgm$^2$), ${K}_P =25  I$, $  K_D=10  I$, $L=0.001$ sec,
${Q =I}$, $\epsilon=5\times10^{-5}$, $\alpha=4,\underline{\beta}_j=0.01$, $\hat{\beta}_j(0)=0.01$, $j=0,1$. {For parity in the  comparison, same values of $\bar{ M}$, $ K_P, K_D, L$ and $\alpha, \gamma_0, \underline{\beta}_0$ are selected for the TDE and ATDE (\ref{gain_old}). }

For all scenarios, a ground reaction force (GRF) is created following the model \cite{6246919} which acts as impulsive external disturbance whenever the foot lands on the ground while walking (cf. Fig. \ref{fig:grf}) and it is considered to be unknown for all the controllers. Due to structural symmetry in Ojas, only the results for the right leg are presented to avoid repetition.

\subsubsection{Description of Scenario S1}
This scenario tests the capability of various controllers to adapt to the dynamic changes in the desired walking trajectory under the GRF (cf. Fig. \ref{fig:grf}), which creates a significant nonlinearity while propagating throughout the body.

\textit{Results and Discussion for S1:} 
The tracking performance of TDE, ATDE and the proposed ARTDE are demonstrated via Fig. \ref{fig:track_s1} and further collected in Table \ref{table 3} in terms of root mean squared (RMS) error, maximum absolute error (MAE) and RMS $ \tau$. Spikes can be noted in every error plots whenever the GRF of the two legs overlap around their peaks (cf. Figs. \ref{fig:grf} and \ref{fig:track_s1}): this indicates both the feet are on ground (double support phase) and the state-dependent constraint forces are in effect. The significantly lower peaks for ARTDE (cf. the MAE data in Table \ref{table 3}) and  \emph{minimum performance improvements} of $29.1\%$ and $20.6\%$ in terms RMSE as compared to TDE and ATDE respectively across all the joints, clearly demonstrate its capability to handle state-dependent forces compared to others. A few snapshots of the walking motion using ARTDE is shown in Fig. \ref{fig:S1_snap}.


\begin{table}[!t]
\renewcommand{\arraystretch}{1.0}
\caption{{Performance Comparison for Scenario S1}}
\label{table 3}
\centering
{
{	\begin{tabular}{c c c c c c }
		\hline
		\hline
		\multirow{2}{*}{Joints} & \multicolumn{3}{c}{Controller} & \multicolumn{2}{c}{Performance Improvement} \\
		\cline{2-6}
		& TDE & ATDE  & {Proposed}  & over TDE & over ATDE   \\
		\hline
& \multicolumn{5}{c}{RMS error (degree)} \\ \cline{1-6}
        \hline
		$q_1$ & 0.108& 0.035  & {0.024} &{77.7\% }& {31.4\% }  \\
		$q_2$ &0.653 &0.551  & {0.436} & {29.1\% }& {20.8\% }  \\
		$q_3$ & 2.231& 1.847 & {1.455} & {34.7\% }& {21.2\% }  \\
		$q_4$ & 3.264& 2.671 & {2.072} & {36.5\% }& {22.4\% }  \\
		$q_5$ & 3.254& 2.741 & {2.190} & {32.6\% }& {20.6\% }   \\
		$q_6$ & 1.147& 1.023 & {0.741} & {35.3\% }& {27.5\% }  \\
		\hline
		& \multicolumn{5}{c}{MAE (degree)} \\ 
		\hline
		$q_1$ & 0.173& 0.225  & {0.075} & {56.65\%} & {66.67\%} \\
		$q_2$ & 1.841& 0.945  & {0.746}& {59.48\%}& {21.06\%} \\
		$q_3$ & 2.456& 1.889 & {1.447} & {41.08\%}& {23.40\%} \\
		$q_4$ & 8.7415& 7.483 & {5.407} & {38.15\%}& {27.74\%} \\
		$q_5$ & 8.660& 6.601 & {4.656} & {46.24\%}& {29.47\%} \\
		$q_6$ & 5.976& 6.190 & {2.415} & {59.59\%}& {60.99\%} \\
		\hline
		& \multicolumn{5}{c}{RMS $\tau$ (Nm)} \\
		\cline{1-6}
		$q_1$ & 2.864&  3.076 & {2.826}&1.33\%&8.13\%\\
		$q_2$ & 3.964 & 4.176  & {4.022}&-1.46\%&3.69\%\\
		$q_3$ & 5.965 & 6.003 & {5.883}&1.37\%&2.00\%\\
		$q_4$ & 5.901 & 5.783 & {5.532}&6.25\% &4.34\%\\
		$q_5$ & 7.764 & 7.861  & {7.832}&-2.16\%&-0.90\%\\
		$q_6$ & 3.133 & 3.479 & {3.148}&-0.48\%&9.51\%\\
		\hline
		\hline
\end{tabular}}}
\end{table}
\begin{figure}[!t]
		\centering
		\includegraphics[width=\columnwidth,height=\paperwidth]{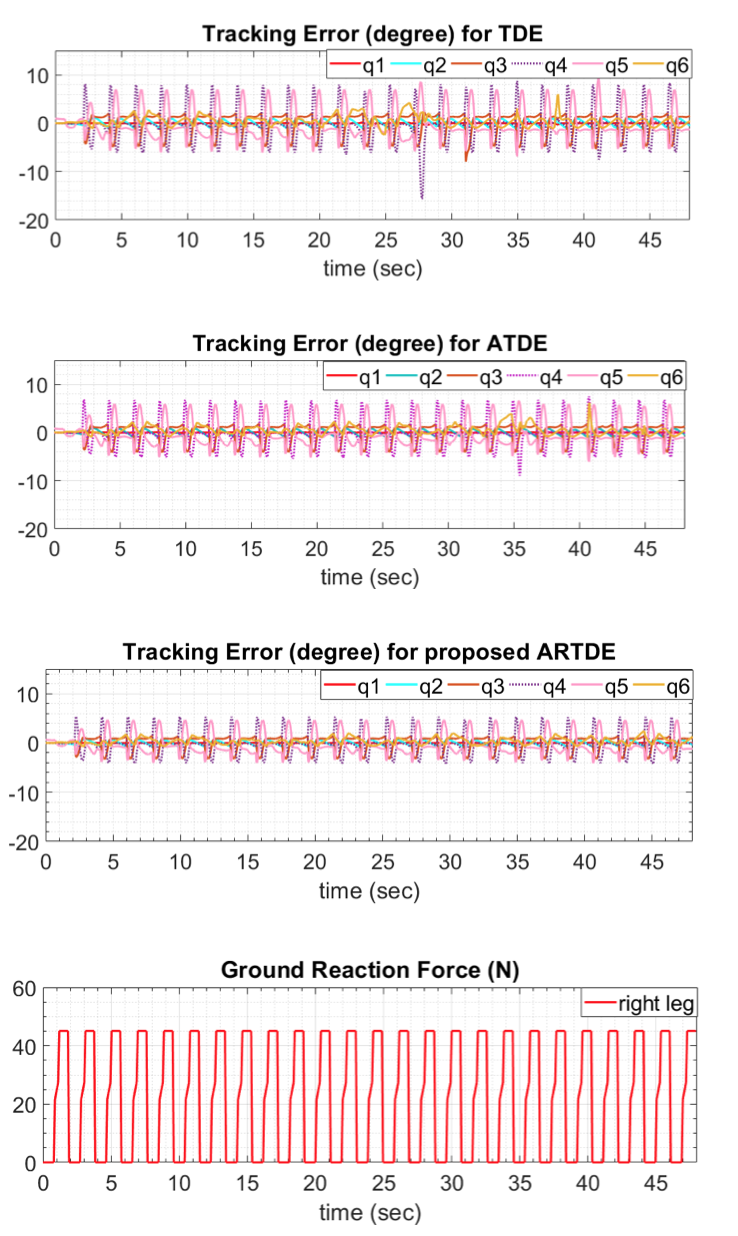}
		\caption{Tracking error for scenario S1.}\label{fig:track_s1}
\end{figure}

\begin{figure}[!t]
		\centering
		\includegraphics[width=1\columnwidth]{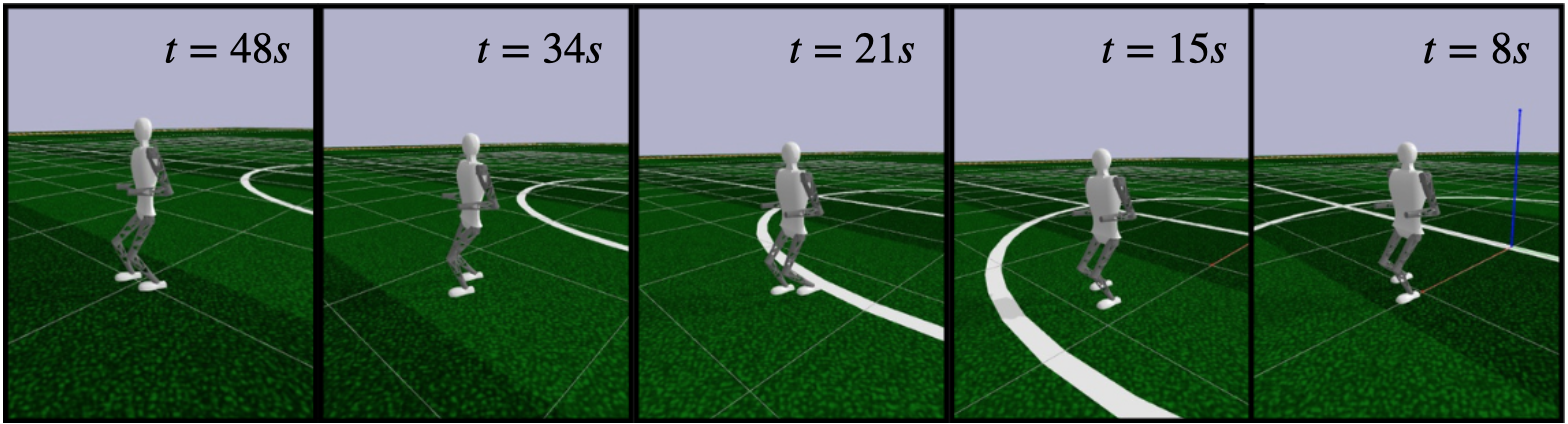}
		\caption{The snapshots from scenario S1 with proposed controller.}
		\label{fig:S1_snap}
\end{figure}

\subsubsection{Description of Scenario S2}
In this second scenario, Ojas is required to follow the same walking motion via Fig. \ref{fig:jointTraj}, but now, while carrying a payload of $1.3$ kg mass (approx. $15\%$ of robot's mass) and using various controllers. 

\begin{figure*}[!h]
		\centering
		\includegraphics[width=\linewidth]{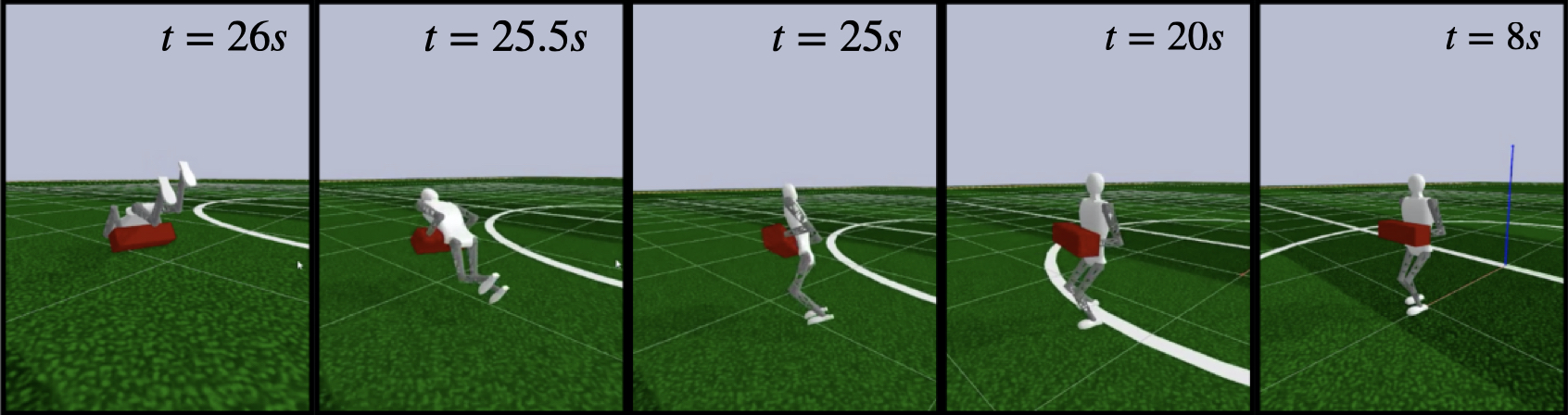}
		\caption{The snapshots from scenario S2 of walking with payload using TDE.}
		\label{fig:s2_snap_TDE}
\end{figure*}

\begin{figure*}[!h]
		\centering
		\includegraphics[width=\linewidth]{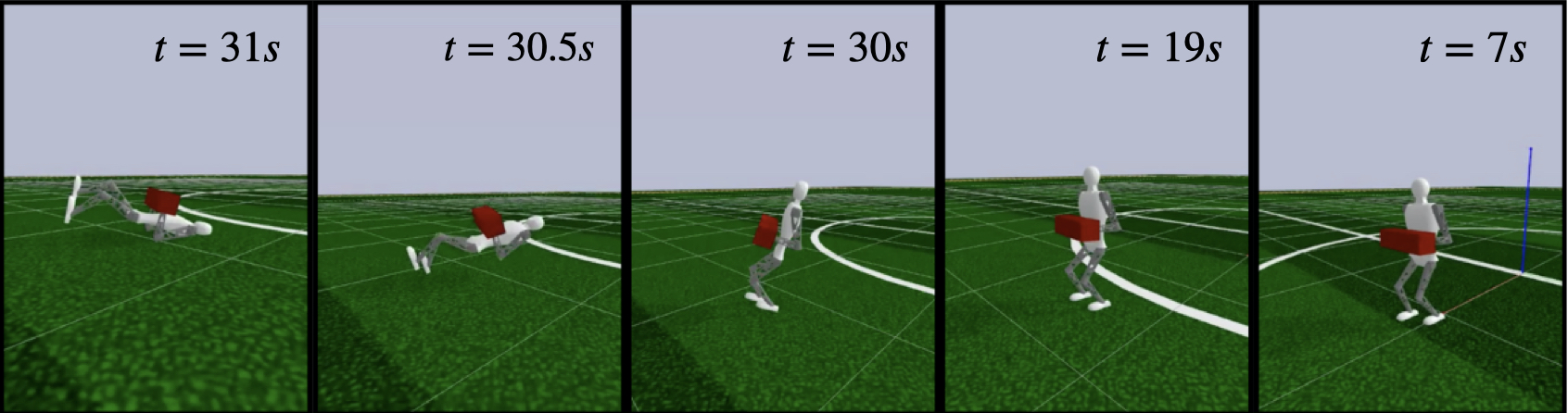}
		\caption{The snapshots from scenario S2 of walking with payload using ATDE.}
		\label{fig:s2_snap_ATDE}
\end{figure*}

\begin{figure*}[!h]
		\centering
		\includegraphics[width=\linewidth]{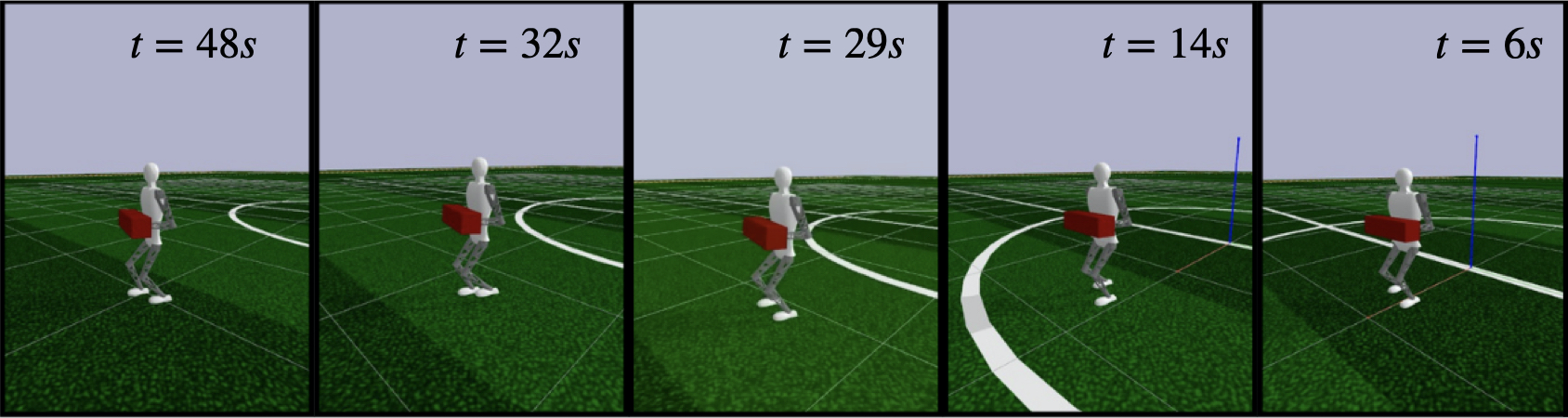}
		\caption{The snapshots from scenario S2 of walking with payload using proposed ARTDE.}
		\label{fig:s2_snap_ARTDE}
\end{figure*}

\textit{Results and Discussion for S2:} 
Unlike scenario S1, Figs. \ref{fig:s2_snap_TDE}-\ref{fig:track_s2} reveal that, while carrying the payload, the robot falls with TDE after $26$ steps (at $t=26$ sec) and with ATDE after $31$ steps (at $t=31$ sec) (cf. the sudden spikes in Fig. \ref{fig:track_s2}). Whereas, the proposed ARTDE could perform the task successfully (cf. Fig. \ref{fig:s2_snap_ARTDE}). This shows that TDE and ATDE, built on the assumption of a priori bounded uncertainty, may fail in presence of state-dependent uncertainty. Since TDE and ATDE have shown destabilizing behaviour, only the performance of ARTDE is given in Table \ref{table 4}.

\begin{figure}[!h]
		\centering
		\includegraphics[width=\linewidth]{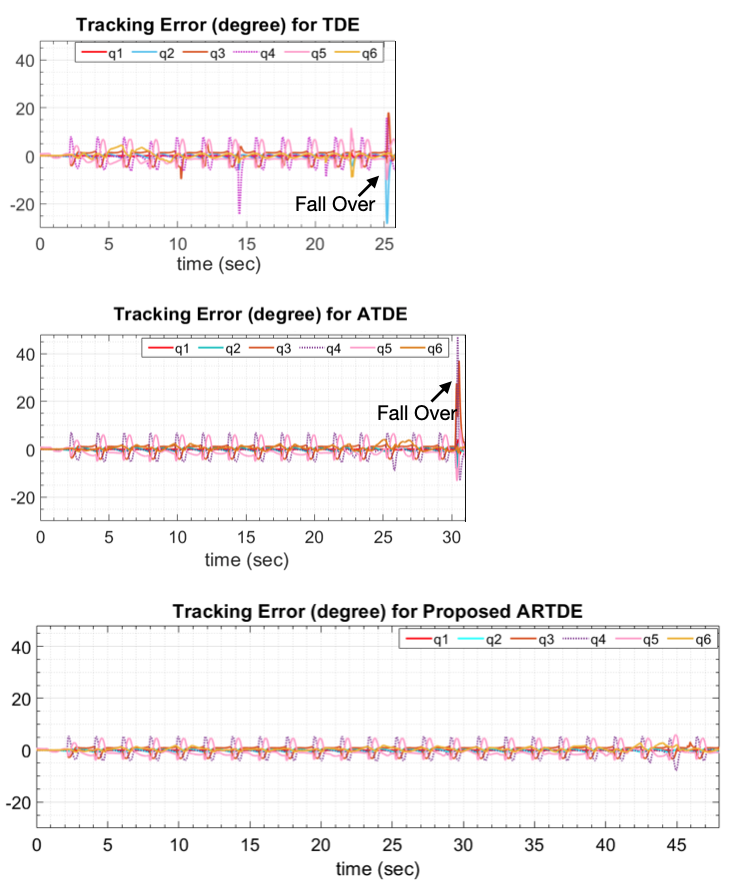}
		\caption{TDE and ATDE fails to stabilize the robot, which falls over at 26s and 31s respectively.} \label{fig:track_s2}
\end{figure}

\begin{table}[!h]
\renewcommand{\arraystretch}{1.0}
\caption{{Performance of the proposed controller under scenario S2}}
\label{table 4}
\centering
{
{	\begin{tabular}{c c c c }
		\hline
		\hline
		{Joints}
		& {RMS error (degree)} & {MAE (degree)} &  {RMS $\tau$ (Nm)}\\
		\hline
		\hline
        \cline{1-4}
		$q_1$ &0.026 &0.063&2.826 \\
		$q_2$ &0.436 & 0.747&4.006\\
		$q_3$&1.466 & 3.140&5.882\\
		$q_4$ &2.120 &5.407&5.657 \\
		$q_5$ &2.205 &5.944&7.974 \\
		$q_6$ &0.777 &2.795&3.253\\
		\hline
		\hline
\end{tabular}}}
\end{table}

\subsubsection{Description of Scenario S3}
To further verify the robustness property of the proposed ARTDE, this scenario is constructed with the following phases: 
\begin{itemize}
    \item[(i)]In phase 1 ($0 \leq t <28$), Ojas walks under similar condition of scenario S2. 
    \item[(ii)]In phase 2 ($28 \leq t <37$), an external impulsive push of $10$ N is applied on the chest (cf. second snapshot in Fig. \ref{fig:s3_snap}) at $t=28$ sec while it was walking. 
    \item[(iii)]In phase 3 ($t \geq 37$), another impulsive push of $10$ N is applied at $t=37$ sec (cf. fourth snapshot in Fig. \ref{fig:s3_snap}), but now to the left arm ($45^{0}$ to the $z$ axis).
\end{itemize}
\textit{Results and Discussion for S3:} The tracking performance of ARTDE as in Fig. \ref{fig:track_s3} clearly highlights the robustness of the proposed design against external disturbances. Further, comparison of RMS error in Tables \ref{table 4} and \ref{table 5} highlights ARTDE provides good repeatability, while higher MAE in $q_2$ and $q_4$ joints stem from the impulsive pushes.

  \begin{figure}[!h]
		\centering
		\includegraphics[width=\linewidth]{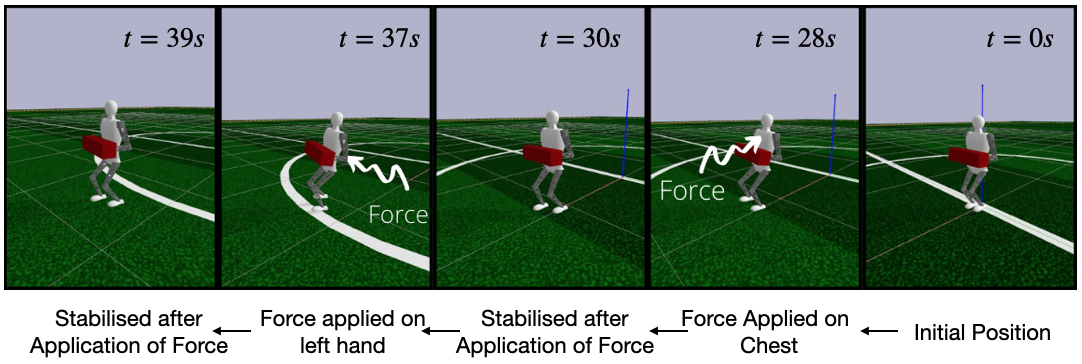}
		\caption{Snapshots from scenario S3.}
		\label{fig:s3_snap}
\end{figure}

\begin{figure}[!h]
		\centering
		\includegraphics[width=\linewidth]{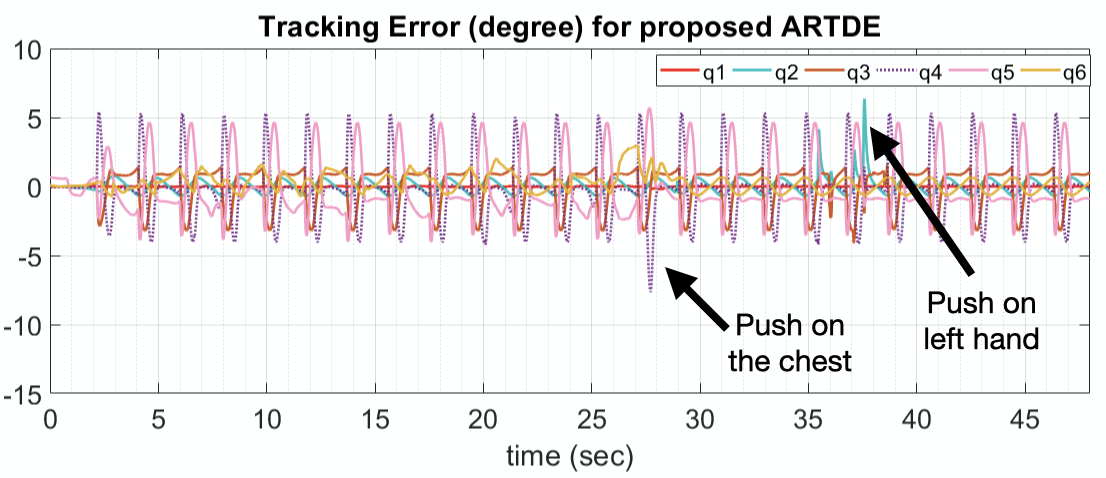}
		\caption{Tracking error of scenario S3.} \label{fig:track_s3}
\end{figure}

\begin{table}[!t]
\renewcommand{\arraystretch}{1.0}
\caption{{Performance of the proposed controller under scenario S3}}
\label{table 5}
\centering
{
{	\begin{tabular}{c c c c c c c}
		\hline
		\hline
		 &\multicolumn{3}{c}{RMS error (degree)}& \multicolumn{3}{c}{MAE (degree)}  \\
		\hline
		{Joints} 
		& Phase1 & Phase2  & Phase3& Phase1 & Phase2  & Phase3  \\
		\hline
		$q_1$ & 0.025& 0.026  & 0.025 & 0.065& 0.062  & 0.063  \\
		$q_2$ & 0.436& 0.441  & 0.497 & 0.745& 4.118  & 6.723\\
		$q_3$ & 1.460& 1.456 & 1.462  & 3.142& 4.008 & 3.802\\
		$q_4$ & 2.118& 2.320 & 2.108  & 5.361& 7.864 & 5.853\\
		$q_5$ & 2.184& 2.243 & 2.182&5.980 & 5.801 & 5.701 \\
		$q_6$ & 0.763& 0.862 & 0.770 & 2.008 & 3.061 & 2.893 \\
		\hline
        
     & \multicolumn{3}{c}{RMS $\tau$ (Nm)} \\
        \cline{1-4}
		\hline
		$q_1$ & 2.824 & 2.912 &2.802 \\
		$q_2$ & 3.924 & 4.004  & 4.964 \\
		$q_3$ & 5.863 & 5.912 & 5.904\\
		$q_4$ & 5.351 & 5.801 &5.368\\
		$q_5$ & 7.78 & 7.981  & 7.763 \\
		$q_6$ & 4.253 & 3.837 & 3.568 \\
		\hline
		\hline
\end{tabular}}}
\end{table}

%% file: 04_Quadrotor_new.tex
\chapter{Adaptive Artificial Time Delay Control for Quadrotors under State-dependent Unknown Dynamics}
\label{quad_chap}

The simplicity in implementation and the significantly low computation burden helped TDE-based methods to find remarkable acceptance in the control literature of robotics in the past decade \cite{Ref:jin2017model,Ref:roy2017adaptive, pi2019adaptive, brahmi2018adaptive, lee2019adaptive, lim2019delayed, roy2020new} including in quadrotors \cite{lee2012experimental, wang2016model, dhadekar2021robust}, showing improved performances compared to conventional methods of robust and adaptive control. Yet, a formal stability analysis of TDE-controlled quadrotors with state-dependent uncertainities is missing. Therefore, a relevant question arises whether \textit{the existing TDE methods can tackle the unknown state-dependent uncertainties in quadrotors}, as, left unattended, such uncertainties can cause instability \cite{kocer2018centralized}. Unfortunately, the TDE methods for quadrotors \cite{lee2012experimental, wang2016model, dhadekar2021robust} (and relevant references therein) rely on the assumption of a priori bounded approximation error (a.k.a. \textit{TDE} error), which is quite common in TDE literature \cite{ Ref:jin2017model,Ref:roy2017adaptive, pi2019adaptive, lee2019adaptive, brahmi2018adaptive, lim2019delayed} and hence, these methods are conservative for quadrotors to negotiate state-dependent uncertainties (please refer to Remark 4 later). Further, being an underactuated system, the adaptive TDE works \cite{Ref:jin2017model,Ref:roy2017adaptive, pi2019adaptive, lee2019adaptive, brahmi2018adaptive, roy2020new} are not directly applicable to a quadrotor. 

In light of the above discussions, artificial delay based adaptive control for quadrotors, which is also robust against \textit{unknown state-dependent uncertainty} is still missing. In this direction, this work has the following major contributions:

\begin{itemize}
\item The proposed adaptive TDE method for a quadrotor system, to the best of the authors' knowledge, is a first of its kind, because the existing TDE methods for quadrotors \cite{lee2012experimental, wang2016model, dhadekar2021robust} are non-adaptive solutions.
    \item Unlike the existing adaptive TDE solutions \cite{Ref:jin2017model,Ref:roy2017adaptive, pi2019adaptive, brahmi2018adaptive, lee2019adaptive}, the proposed adaptive TDE method can provide \emph{robustness} (hence termed as adaptive-robust TDE, ARTDE) against state-dependent unmodelled dynamics. 
    \item The closed-loop system stability is established analytically. Further, experimental results suggest significant improvement in tracking accuracy for the proposed method compared to the state of the art.
\end{itemize}

The rest of the chapter is organised as follows: Section 3.2 describes the dynamics of quadrotor;
Section 3.3 details the proposed control framework, while corresponding stability analysis is provided
in Section 3.4; comparative simulation results are provided in Section 3.5.

\section{Quadrotor System Dynamics}
\begin{figure}[!h]
    \includegraphics[width=2 in, height=2in]{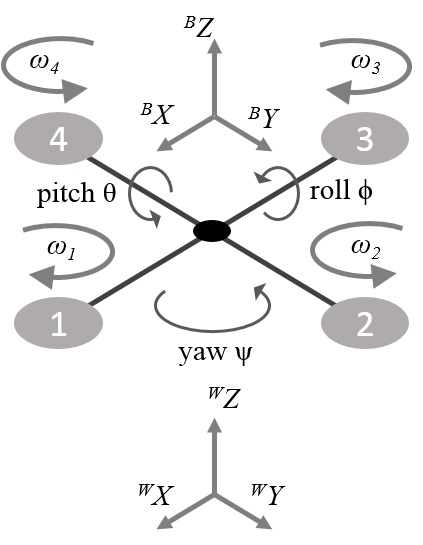}
    \centering
    \caption{Schematic of a quadrotor and the coordinate frames.}
    \label{fig:quad_axis}
\end{figure}
	The well established Euler-Lagrangian system dynamics of a quadrotor model (cf. Fig. \ref{fig:quad_axis}) is given by \cite{bialy2013lyapunov}
	\begin{subequations}\label{dyn}
	\begin{align}	
	&	m \ddot{p}(t) + G +  d_p(p(t), \dot{p}(t),t) = \tau_p(t), \label{p_tau} \\
	&	J (q(t))\ddot{q}(t) + C_q(q(t),\dot{q}(t))\dot{q}(t) +  d_q(q,\dot{q},t) = \tau_q(t), \label{q_tau} \\
	&	\tau_p (t) = R_B^W (q)U(t), \label{tau_conv}
	\end{align}
	\end{subequations}
	where (\ref{p_tau}) and (\ref{q_tau}) represent the position dynamics and the attitude dynamics, respectively; (\ref{tau_conv}) converts the input vector $\tau_p \in \mathbb{R}^3$ in Earth-fixed frame to $U\triangleq
	\begin{bmatrix}
	0 & 0 & u_1
	\end{bmatrix}^T\in \mathbb{R}^3$ in body-fixed frame via the $Z$-$Y$-$X$ Euler angle rotational matrix $R_B^W$	given by
	\begin{align}
	R_B^W =
	\begin{bmatrix}
	c_{\psi}c_{\theta} & c_{\psi}s_{\theta}s_{\phi} - s_{\psi}c_{\phi} & c_{\psi}s_{\theta}c_{\phi} + s_{\psi}s_{\phi} \\
	s_{\psi}c_{\theta} & s_{\psi}s_{\theta}s_{\phi} + c_{\psi}c_{\phi} & s_{\psi}s_{\theta}c_{\phi} - c_{\psi}s_{\phi} \\
	-s_{\theta} & s_{\phi}c_{\theta} & c_{\theta}c_{\phi}
	\end{bmatrix}, \label{rot_matrix}
	\end{align}
	where $c_{(\cdot)} , s_{(\cdot)}$ are abbreviations for $\cos{(\cdot)} , \sin{(\cdot)}$ respectively. Various other symbols in (\ref{dyn}) are described as follows: the mass and inertia matrix are represented by $m \in \mathbb{R}^{+}$ and $J(q) \in \mathbb{R}^{3 \times 3}$ respectively; the center-of-mass of the quadrotor is denoted by the position vector $p\triangleq
	\begin{bmatrix}
	x & y & z
	\end{bmatrix}^T \in\mathbb{R}^3$; the orientation/attitude (roll, pitch, yaw angles respectively) is denoted via $q\triangleq \begin{bmatrix}
	\phi & \theta & \psi
	\end{bmatrix}^T\in \mathbb{R}^3$; $G\triangleq
	\begin{bmatrix}
	0 & 0 & mg
	\end{bmatrix}^T\in \mathbb{R}^3$ denotes the gravitational force vector with gravitational constant $g$; $C_q(q, \dot{q}) \in \mathbb{R}^{3 \times3}$ is the Coriolis matrix; the unmodelled disturbances, which can be both state and time dependent, are denoted by $d_p$ and $d_q$; $\tau_q \triangleq 
	\begin{bmatrix}
	u_2 & u_3 & u_4
	\end{bmatrix}^T\in\mathbb{R}^3$ are the control inputs for roll, pitch and yaw; 
	
	From the standard Euler-Lagrange mechanics, the following property holds \cite{bialy2013lyapunov}:
	
\noindent\textbf{Property 1.} The inertia matrix $J (q)$ is uniformly positive definite $\forall q$. 

In the following, we highlight the available system parametric knowledge for control design:
\begin{assum}
The exact knowledge of $m,J$ is not available, and only some upper bounds are known (cf. Remark 6 later); meanwhile, the system dynamics term $C_q(q,\dot{q})$, and unmodelled dynamics $d_p,d_q$ and their bounds are unknown for control design. The terms $d_p,d_q$ can have state-dependency, and hence cannot be bounded a priori.
\end{assum}
\begin{remark}[Validity of Assumption 1]
In practice, maximum allowable payload for a quadrotor is always available; therefore, a priori upper bound knowledge of $m$ and $J$ is plausible for control design, while handling other unknown state-dependent dynamics terms is a control challenge solved in this work. 
\end{remark}

We further take the following standard assumption:
\begin{assum}[\cite{bialy2013lyapunov, mellinger2011minimum, tang2015mixed}]\label{assum_des}
The desired position $p_d  \triangleq
	\begin{bmatrix}
	x_d  & y_d  & z_d
	\end{bmatrix}^T$ and yaw trajectories $\psi _d $ are designed such that they are smooth and bounded. 
\end{assum}
\begin{remark}[Desired roll and pitch]
As clarified in Sect. III.B and also in standard literature \cite{mellinger2011minimum}, the desired roll ($\phi_d$) and pitch $(\theta_d$) angle trajectories are computed using $\tau_p$ and $\psi_d$. 
\end{remark}
\section{Proposed Controller Design and Analysis}
%

\subsection{Control Problem and Objective} Under Property 1 and Assumption 1, the aim is to design an adaptive robust TDE (ARTDE) controller for quadrotors to track a desired trajectory (cf. Assumption 2). 

 The position and attitude co-design approach relies on simultaneously designing an outer loop controller for (\ref{p_tau}) and of an inner loop controller for (\ref{q_tau}). Following this approach (cf. Fig. \ref{fig:block_dia} later), the proposed ARTDE framework is elaborated in the following subsections along with stability analysis. 
 \subsection{Outer Loop Controller}
Let the position tracking error be defined as ${e_p}(t)={p}^d(t)-{p}(t)$. The variable dependency will be removed subsequently for brevity whenever it is obvious.
Before presenting the outer loop controller, the position dynamics (\ref{p_tau}) is re-arranged by introducing a constant ${\bar {m}}$ as
\begin{align} 
&{\bar{m}\ddot {p}} + {{N_p}}({{p}},{{\dot { p}},\ddot{ p}}) = {\boldsymbol {\tau}_p }, \label{robotdynamics2}\\
\text{with}~~~~&{{N_p}}({{p}},{{\dot { p}},\ddot{ p}}) = ({{m}} - {\bar {m}}){\ddot {p}} +  G + {d_p}\nonumber
\end{align}
and the selection of ${\bar {m}}$ is discussed later (cf. Remark \ref{rem_up}). Note that, via Assumption 1, $N_p$ subsumes the unknown/unmodelled dynamics, and it is considered to depend on states $(p, \dot{p})$ via $d_p$.

The control input $\boldsymbol{\tau}_p$ is proposed as
\begin{subequations}\label{tau}
\begin{align}
\boldsymbol{\tau}_p &=\bar{m} u_p+\hat{ N}_p({{p}},{{\dot {p}},\ddot{p}}), \label{input}\\
{u}_p &= {u_p}_0+\boldsymbol\Delta u_p,\label{tarc input}\\
{u_p}_0 &=\ddot{p}^d +  K_{1p}\dot{e}_p + K_{2p}  e_p, \label{aux}
\end{align}
\end{subequations}
where $K_{2p}, K_{1p} \in {R}^{3 \times 3}$ are two user-defined positive definite matrices; $\boldsymbol\Delta u_p$ is the adaptive control term responsible to tackle uncertainties in position dynamics and it will be designed later; $\hat{ N}_p$ is the estimation of $N_p$ derived from the past state and input data as
\begin{equation}\label{approx_p}
{\hat{ N}_p( p,\dot{ p} , \ddot{ p})\cong}N_p( p_L,\dot{ p}_L,\ddot{p}_L)=\boldsymbol(\tau_p)_L-\bar{ m}\ddot{ p}_L,
\end{equation}
where $L>0$ is a small time delay which and its choice is discussed later. The notation $(\cdot)_L=(\cdot)(t-L)$ is defined at the end of Sect. I.
\begin{remark}[Artificial time delay]\label{rem_tde}
The estimation of uncertainty as in (\ref{approx_p}) is based on intentionally (or artificially) introducing a time delay (a.k.a TDE) in the form of past data; some literature (cf. \cite{Ref:9_1, Ref:jin,Ref:jin2017model,Ref:roy2017adaptive, pi2019adaptive, wang2018new, lee2019adaptive, brahmi2018adaptive, lim2019delayed} and referenced therein) terms this mechanism as artificial time delay based design.
\end{remark}

Substituting (\ref{input}) into (\ref{robotdynamics2}), one obtains
\begin{align}
\ddot{e}_p& =- K_{1p}\dot{ e}_p- K_{2p} e_p+\boldsymbol \sigma_p -\boldsymbol\Delta u_p , \label{error dyn delayed 2}
\end{align}
where $\boldsymbol\sigma_p={{\frac{1}{\bar{m}}}}({ N_p}-{\hat{ N}_p})$ is the \textit{estimation error originating from (\ref{approx_p})} and it is termed as the \textit{TDE error}.

The adaptive control term $\Delta \mathbf u_p$ is designed based on the structure of the upper bound of TDE error $\boldsymbol\sigma_p$. Therefore, in the following, we first derive the upper bound structure of $|| \boldsymbol\sigma_p ||$ and subsequently design the proposed adaptive law.

\subsubsection{Upper bound structure of $\boldsymbol \sigma_p$} \label{sec bound}
From (\ref{robotdynamics2}) and (\ref{error dyn delayed 2}), once can derive the following:
\begin{align}
\hat{N}_p&=(N_p)_L=[{{m}}({{p}_L}) - {\bar { m}}]{\ddot { p}_L} + G_L + (d_p)_L ,\label{sig 2} \\
\boldsymbol \sigma_p&=\ddot{ p}- u_p. \label{sig 1} 
\end{align}
Using (\ref{sig 2}), the control input $\boldsymbol \tau_p$ in  (\ref{input}) can be rewritten as
\begin{align}
\boldsymbol {\tau}_p &= \bar{m} u_p+[{ {m}}({{p}_L}) - {\bar {m}}]{\ddot { p}_L} +  G_L + (d_p)_L. \label{tau new}
\end{align}
Multiplying both sides of (\ref{sig 1}) with $m$ and using (\ref{robotdynamics2}) and (\ref{tau new}) we have
\begin{align}
{m} \boldsymbol \sigma_p  &= \boldsymbol {\tau}_p  - N_p-{m}{u_p} \nonumber \\
& = \bar{m}  u_p+[{{m}}({{p}_L}) - {\bar { m}}]{\ddot {p}_L} + G_L + (d_p)_L -N_p -{m}{u_p}. \label{sig 3}
\end{align}
Defining $K_p \triangleq [K_{1p} ~ K_{2p}]$,$\xi_L= \begin{bmatrix}
{ e_L}^T&
\dot{e_L}^T\end{bmatrix}^T$ and using (\ref{error dyn delayed 2}) we have
\begin{align}
\ddot{{p}}_L&=\ddot{{p}}^d_L-\ddot{{e}}_L =\ddot{{p}}^d_L +K_p\boldsymbol \xi_L-\boldsymbol \sigma_L+(\boldsymbol\Delta  {u_p})_L.\label{sig 4}
\end{align}
Substituting (\ref{sig 4}) into (\ref{sig 3}), and after re-arrangement yields
\begin{align}
\boldsymbol \sigma_p &= \underbrace{m^{-1}\bar{{m}}(\boldsymbol \Delta u_p- \boldsymbol (\Delta  u_p)_L)}_{\boldsymbol \chi_1}+\underbrace{m^{-1}({m}_L\boldsymbol (\Delta u_p)_L-{m}\boldsymbol (\Delta u_p)_L)}_{\boldsymbol \chi_2} \nonumber\\
&+\underbrace{m^{-1}\lbrace (\bar{{m}} - {m})\ddot{{p}}^d-({m} - {m}_L)\ddot{{p}}^d_L+ G_L + (d_p)_L - G - d_p \rbrace }_{\boldsymbol \chi_3} \nonumber\\
& +\underbrace{m^{-1}({m}_L-\bar{{m}}){K}_p\boldsymbol \xi_L}_{\boldsymbol \chi_4}-\underbrace{m^{-1}({m}_L-\bar{{m}})\boldsymbol \sigma_L}_{\boldsymbol \chi_5} \nonumber\\
&+\underbrace{(\bar{{m}}/m-1){K}_p\boldsymbol \xi_p}_{\boldsymbol \chi_6}. \label{sig 5}
\end{align}

Any function $\psi$ delayed by time $L$ 
can be represented as
\begin{equation}
\psi_L= 
\psi(t)-\int_{-L}^0 {\frac{\mathrm{d}}{\mathrm{d}\theta}\psi}(t+\theta)\mathrm{d}\theta. \label{delay}
\end{equation}
Since integration of any continuous function over a finite interval (here $-L$ to $0$) is always finite \cite{Ref:rudin}, and using (\ref{delay}), the following conditions are satisfied for unknown constants $\delta_i$, $i=1, \cdots,5$:
\begin{align}
&|| {\boldsymbol \chi_1} ||=|| \frac{1}{m}\bar{{m}} \int_{-L}^{0} \frac{{d}}{{d}\theta}\boldsymbol \Delta { u_p}(t+\theta) {d}\theta || \leq \delta_1 \label{sig 6}\\
&|| {\boldsymbol \chi_2} ||=||\frac{1}{m}\hspace{-.1cm} \int_{-L}^{0} \frac{{d}}{{d}\theta}{m}(t+\theta)\boldsymbol \Delta {u}_p(t+\theta) {d}\theta || \leq \delta_2\\
&|| {\boldsymbol \chi_3} ||=||\frac{1}{m}\lbrace (\bar{{m}} -m)\ddot{{p}}^d-({m}-{m}_L)\ddot{{p}}^d_L \nonumber\\
&\qquad \qquad \qquad  G_L + (d_p)_L -G -d_p \rbrace ||\leq \delta_3 \\
&|| {\boldsymbol \chi_4} ||=|| \frac{1}{m} \int_{-L}^{0} \frac{{d}}{{d}\theta}({m}(t+\theta)-\bar{{m}}){K}_p\boldsymbol \xi_p(t+\theta) {d}\theta  \nonumber \\
& \qquad \qquad +(\bar{{m}}/m-1){K}_p\boldsymbol \xi_p  ||\leq || E_p  K_p ||  ||\boldsymbol \xi_p || +\delta_4 \\
&|| {\boldsymbol \chi_5} ||=|| E_p \boldsymbol \sigma_p +  \frac{1}{m} \int_{-L}^{0} \frac{{d}}{{d}\theta}\lbrace({m}(t+\theta)-\bar{{m}}) \boldsymbol \sigma_p (t+\theta) \rbrace {d}\theta || \nonumber \\
& \qquad \leq ||{ E_p}|| ||\boldsymbol \sigma_p ||  + \delta_5  \label{del_5}\\
&|| {\boldsymbol \chi_6} ||= || (\bar{{m}}/m-1){K}_p\boldsymbol \xi_p || \leq || {E_p}{K}_p|| ||\boldsymbol \xi_p|| . \label{sig 7}
\end{align} 
where the following holds
\begin{equation}
   | E_p |= | 1 -\bar{m}/m | <1. \label{E} 
\end{equation}
Using (\ref{sig 5}) and (\ref{sig 6})-(\ref{sig 7}), one derives
\begin{align}
 \lVert \boldsymbol \sigma_p \rVert &\leq \beta_{0p} + \beta_{1p} \lVert\boldsymbol \xi_p \rVert, \label{sig bound} \\
 \text{where}~ &\beta_{0p}=\frac{\sum_{i=1}^{5}\delta_i}{1-| E_p |}, ~ \beta_{1p}=\frac{2 \lVert E_p  K_p \rVert  }{1-| E_p|} \label{beta}.
\end{align}
\begin{remark}[State-dependent TDE error bound]\label{rem_new}
Note from (\ref{sig bound}) that the TDE error $\boldsymbol \sigma_p$ depends on states via $\boldsymbol \xi_p$ and hence cannot be bounded a priori: for this reason, the standard TDE/ adaptive TDE methods \cite{Ref:9_1, Ref:jin,Ref:jin2017model,Ref:roy2017adaptive, pi2019adaptive, wang2018new, brahmi2018adaptive, lim2019delayed} are not applicable for quadrotors as they rely on a priori bounded TDE error (i.e., they assume $\beta_{1p}=0$). Similar case can be noted for attitude dynamics as well (cf. (\ref{sig bound_q}) later). Therefore, a new and suitable adaptive law $\Delta u_p$ is derived subsequently.  
\end{remark}

\begin{remark}[Choice of $L$]
It can be noted from (\ref{sig 6})-(\ref{del_5}), that high value of time delay $L$ will lead to high values of $\delta_i$, i.e., larger TDE error: therefore, one needs to select the smallest possible $L$, which is usually selected as the sampling time of the low level micro-controller \cite{Ref:9_1, Ref:jin,Ref:jin2017model,Ref:roy2017adaptive, pi2019adaptive, wang2018new, brahmi2018adaptive, lee2019adaptive, lim2019delayed, lee2012experimental, wang2016model, dhadekar2021robust}.
\end{remark}


\subsubsection{Designing $\Delta u_p$} \label{sec gain}
The term $\boldsymbol\Delta u_p$ is designed as
\begin{align}
&\boldsymbol\Delta {{u}}_p =\alpha_p c_p \frac{s_p}{||s_p||},  \label{delta u2_p} 
\end{align}
where ${s_p=B}^T{U_p}\xi_p$, $\xi_p= \begin{bmatrix}
{ e_p}^T&
\dot{e_p}^T
\end{bmatrix}^T$ and $U_p$ is the solution of the Lyapunov equation ${A_p}^TU_p+{U_p}{A_p}=-Q_p$ for some ${Q_p>0}$, where $  A_p=\begin{bmatrix}
{0} & {I} \\
-{K}_{2p} & -{K}_{1p}
\end{bmatrix}$, ${B}=\begin{bmatrix}
{0}\\
{I}
\end{bmatrix}$; $\alpha_p \in \mathbb{R}^{+}$ is a user-defined scalar; ${c_p} \in \mathbb{R}^{+}$ is the overall switching gain which provides robustness against the TDE error. 

The gain $c_p$ in (\ref{delta u2_p}) is constructed from the upper bound structure of $||\boldsymbol \sigma_p ||$ as
\begin{align}
c_p=\hat{\beta}_{0p}+\hat{\beta}_{1p} ||\boldsymbol \xi_p || , \label{sw gain} 
\end{align}
where $\hat{\beta}_{0p}, \hat{\beta}_{1p}$ are the estimates of $\beta_{0p}, \beta_{1p} \in \mathbb{R}^{+}$, respectively. 
The gains are evaluated as follows: 
\begin{align}
\dot{\hat{\beta}}_{ip} =
  \begin{cases}
     \lVert\boldsymbol \xi_p \rVert^i \lVert  s_p \rVert,       &  {\text{if} ~ \text{any}~\hat{\beta}_{ip} \leq \underline{\beta}_{ip} ~ \text{or} ~{s}_p^T \dot{{s}}_p >0} \\
    -  \lVert\boldsymbol \xi_p \rVert^i \lVert  s_p \rVert,       & {\text{if} ~ {s}_p^T \dot{{s}}_p  \leq 0 ~\text{and all}~  \hat{\beta}_{ip} > \underline{\beta}_{ip} }
  \end{cases}, \label{ATRC_p} 
\end{align}
\noindent with $\hat{\beta}_{ip}(0) \geq \underline{\beta}_{ip} >0$, $i=0,1$ are user-defined scalars. Combining (\ref{input}), (\ref{tarc input}), (\ref{approx}) and (\ref{delta u2}), we have
\begin{equation}\label{theproposedcontrollaw}
\begin{split}
{\boldsymbol{\tau_p }} = & \underbrace {({\boldsymbol{\tau }_p)_L} - {\bar{m}}{{{\ddot { p}}}_L}}_{{\text{TDE~part}}}+\underbrace {{\bar{ m}}({{{\ddot { p}}}^d} + K_{1p}{\dot e}_p + K_{2p}{e_p})}_{{\text{Desired~dynamics~injection~part}}}\\
&+\underbrace { {\bar{ m}} c_p (s_p/||s_p||) .}_{{\text{Adaptive-robust~control~part}}} 
\end{split}
\end{equation}
Note that eventually the actual input $U$ is applied to the system using $\tau_p$ and the transformation as in (\ref{tau_conv}). 
\begin{figure}[!h]
    \centering
    \includegraphics[width=1\columnwidth]{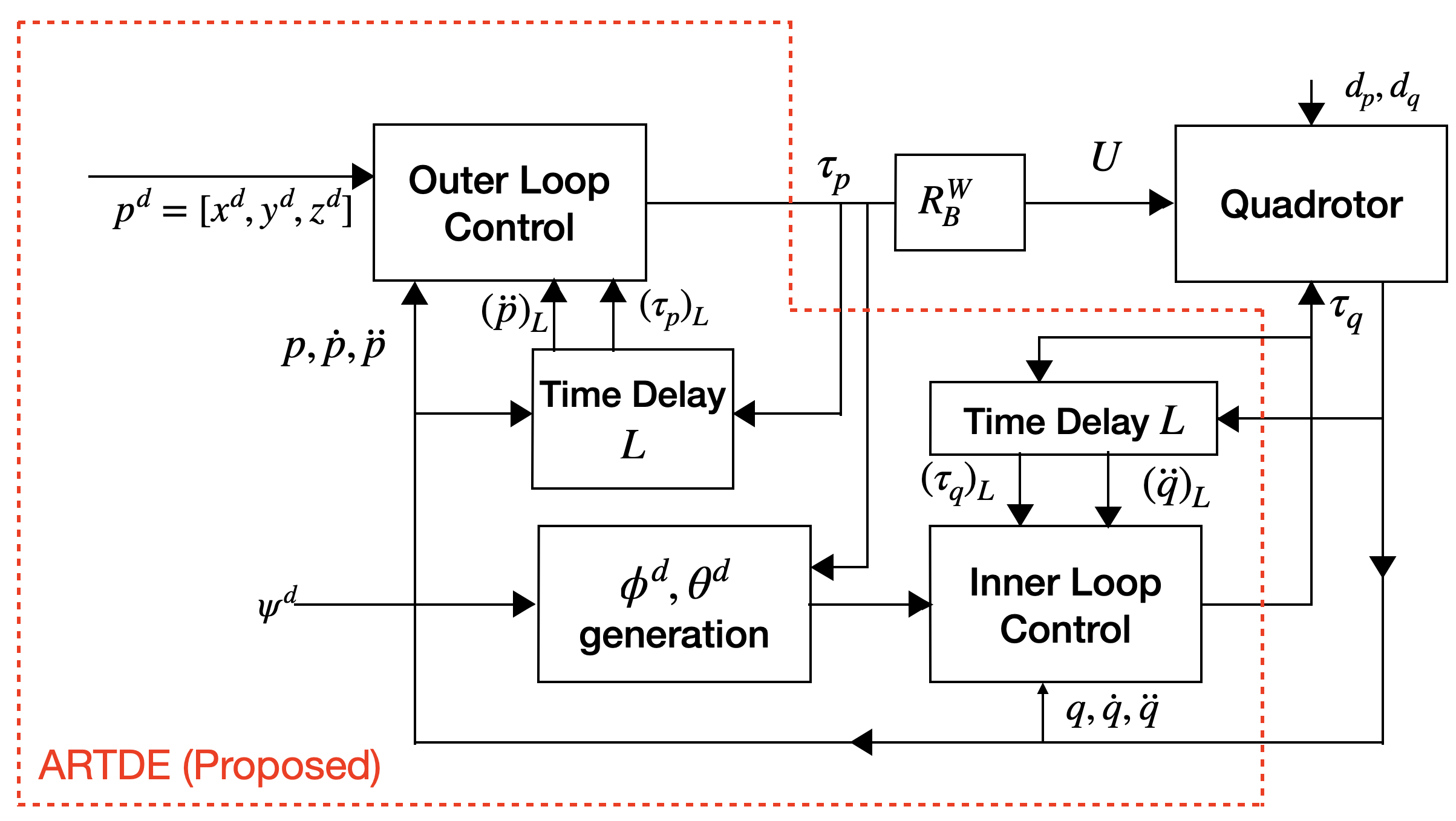}
    \centering
    \caption{Block diagram of the proposed ARTDE control framework via the outer- and inner loop co-design approach.}
    \label{fig:block_dia}
\end{figure} 
\subsection{Inner Loop Controller}
For designing the inner loop controller, the desired roll ($\phi_d$) and pitch ($\theta_d$) angles are first to be computed via defining an intermediate coordinate frame as (cf. \cite{mellinger2011minimum}):
\begin{subequations}\label{int_co}
\begin{align}
    z_B &= \frac{\tau_p}{||\tau_p||},~y_A = \begin{bmatrix}
    -s_{\psi_d} & c_{\psi_d} & 0
\end{bmatrix}^T \\
    x_B &= \frac{y_A \times z_B}{||y_A \times z_B||},~y_B = z_B \times x_B
\end{align}
\end{subequations}
where $(x_B,y_B,z_B)$ denote the $(x, y, z)$-axis of the body-fixed coordinate frame and $y_A$ is the $y$-axis of the intermediate coordinate frame $A$. Based on the given desired yaw angle $\psi_d (t)$ and the derived intermediate axes as in (\ref{int_co}), $\phi_d (t)$ and $\theta_d (t)$ can be computed using the desired body-fixed frame axes as described in \cite{mellinger2011minimum} (omitted due to lack of space). 

Further, the orientation/attitude error is defined as \cite{mellinger2011minimum}
\begin{align}
    e_q &= {((R_d)^T R_B^W - (R_B^W)^T R_d)}^{v} \label{q_err} \\
    \dot{e}_q & = \dot{q} - R_d^T R_B^W \dot{q}_d
\end{align}
where $R_d$ is the rotation matrix as in (\ref{rot_matrix}) evaluated at ($\phi_d, \theta_d, \psi_d$) and $(.)^v$ is the \textit{vee} map converting elements of $SO(3)$ to $\in{\mathbb{R}^3}$ \cite{mellinger2011minimum}.

Introducing a constant matrix $\bar{J}$ (cf. Remark \ref{rem_up} for its choice), the attitude dynamics (\ref{q_tau}) is re-arranged as 
\begin{align} 
&{\bar{J}\ddot {q}} + {{N_q}}({{q}},{{\dot { q}},\ddot{q}}) = {\boldsymbol {\tau}_q } \label{sys_q},\\
\text{with}~~~~&{{N_q}}({{q}},{{\dot { q}},\ddot{ q}}) = [{{J}} - {\bar {J}}]{\ddot {q}} +  C_q{\dot { q}} + {d_q}.
\end{align}
The control input $\boldsymbol{\tau}_q$ is designed as
\begin{subequations}\label{tau_q}
\begin{align}
\boldsymbol{\tau}_q &=\bar{J} u_q+\hat{ N}_q({{q}},{{\dot {q}},\ddot{q}}), \label{input_q}\\
{u}_q &= {u_q}_0+\boldsymbol\Delta u_q,\label{tarc input_q}\\
{u_q}_0 &=\ddot{q}^d +  K_{1q}\dot{e}_q + K_{2q}  e_q, \label{aux_q}
\end{align}
\end{subequations}
where $K_{1q}, K_{2q} \in {R}^{3 \times 3}$ are two user-defined positive definite matrices; $\boldsymbol\Delta u_q$ is the adaptive control term responsible to tackle uncertainties in attitude dynamics and it would be designed later; $\hat{ N}_q$ is the estimation of $N_q$ computed via TDE as
\begin{equation}\label{approx}
{\hat{ N}_q( q,\dot{ q} , \ddot{ q})\cong}N_q( q_L,\dot{ q}_L,\ddot{q}_L)=\boldsymbol(\tau_q)_L-\bar{ J}\ddot{ q}_L.
\end{equation}
Substituting (\ref{input_q}) into (\ref{sys_q}), we obtain the error dynamics as
\begin{align}
\ddot{e}_q& =- K_{1q}\dot{ e}_q- K_{2q} e_q+\boldsymbol \sigma_q- \boldsymbol\Delta u_q, \label{error dyn delayed 2_q}
\end{align}
where $\boldsymbol\sigma_q={\bar{J}}^{-1}({ N_q}-{\hat{N_q}})$ represents the attitude \textit{TDE error}.

\subsubsection{Upper bound structure of $\boldsymbol \sigma_q$}The upper bound structure for $||\sigma_q||$ can be derived in a similar fashion to that of $||\sigma_p||$ in Sect. \ref{sec bound} for the outer loop control (omitted due to lack of space and to avoid repetition) and it is found as
\begin{align}
 \lVert \boldsymbol \sigma_q \rVert &\leq \beta_{0q} + \beta_{1q} \lVert\boldsymbol \xi_q \rVert, \label{sig bound_q} \\
 \text{where}~ &\beta_{0q}=\frac{\sum_{i=1}^{5}\bar{\delta}_i}{1-\lVert E_q \rVert}, ~ \beta_{1q}=\frac{2 \lVert E_q  K_{q} \rVert  }{1-\lVert E_q\rVert}, K_q \triangleq [K_{1q} ~ K_{2q}] \nonumber 
\end{align}
where $\bar{\delta}_i$ are unknown constants and the following holds
\begin{equation}
   || E_q ||= ||J^{-1}\bar{J}-I|| < 1. \label{E_q} 
\end{equation}
\subsubsection{Designing $\Delta u_q$} The term $\boldsymbol\Delta u_q$ is designed as
\begin{align}
&\boldsymbol\Delta {{u}}_q =\alpha_q c_q\frac{s_q}{||s_q||},  \label{delta u2} 
\end{align}
where ${s_q=B}^T{U_q}\xi_q$, $\xi_q= \begin{bmatrix}
{ e_q}^T&
\dot{e_q}^T
\end{bmatrix}^T$ and $U_q$ is the solution of the Lyapunov equation ${A_q}^TU_q+{U_q}{A_q}=-Q_q$ for some ${Q_q>0}$, where $  A_q=\begin{bmatrix}
{0} & {I} \\
-{K}_{2q} & -{K}_{1q}
\end{bmatrix}$, ${B}=\begin{bmatrix}
{0}\\
{I}
\end{bmatrix}$; $\alpha_q \in \mathbb{R}^{+}$ is a user-defined scalar; ${c_q} \in \mathbb{R}^{+}$ is the overall switching gain designed as
\begin{align}
c_q=\hat{\beta}_{0q}+\hat{\beta}_{1q} ||\boldsymbol \xi_q || , \label{sw gain} 
\end{align}
where $\hat{\beta}_{0q}, \hat{\beta}_{1q}$ are the estimates of $\beta_{0q}, \beta_{1q} \in \mathbb{R}^{+}$, respectively. 
The gains are evaluated as follows: 
\begin{align}
\dot{\hat{\beta}}_{iq} =
  \begin{cases}
   \lVert\boldsymbol \xi_q \rVert^i \lVert  s_q \rVert,       &  {\text{if} ~ \text{any}~\hat{\beta}_{iq} \leq \underline{\beta}_{iq} ~ \text{or} ~{s}_q^T \dot{{s}}_q >0} \\
    -  \lVert\boldsymbol \xi_q \rVert^i \lVert  s_q \rVert,       & {\text{if} ~ {s}_q^T \dot{{s}}_q  \leq 0 ~\text{and all}~  \hat{\beta}_{iq} > \underline{\beta}_{iq} }
  \end{cases}, \label{ATRC} 
\end{align}
\noindent with $\hat{\beta}_{iq}(0) \geq \underline{\beta}_{iq} >0$, $i=0,1$ are user-defined scalars. Finally, the inner loop control becomes
\begin{equation}\label{theproposedcontrollaw_q}
\begin{split}
{\boldsymbol{\tau_q }} = & \underbrace {({\boldsymbol{\tau }_q)_L} - {\bar{J}}{{{\ddot { q}}}_L}}_{{\text{TDE~part}}}+\underbrace {{\bar{ J}}({{{\ddot { q}}}^d} + {{{K}}_D}{\dot { e_q}} + {{{K}}_P}{{e_q}})}_{{\text{Desired~dynamics~injection~part}}}\\
&+\underbrace {{\bar{ J}} c_q (s_q/||s_q||) .}_{{\text{Adaptive-robust~control~part}}} %
\end{split}
\end{equation}
The overall proposed control framework, comprising the simultaneous design of outer- and inner loop control is depicted via Fig. \ref{fig:block_dia}. 
\begin{remark}[On the choice of $\bar{m}$ and $\bar{J}$]\label{rem_up}
Based on the a priori knowledge of the maximum payload carrying capacity of the system, i.e. upper bounds on $m$ and $J$ (cf. Remark 1), one can always design $\bar{m}$ and $\bar{J}$ which satisfy (\ref{E}) and (\ref{E_q}) respectively. Such condition is standard in TDE literature \cite{Ref:9_1, Ref:jin,Ref:jin2017model,Ref:roy2017adaptive, pi2019adaptive, wang2018new, lee2019adaptive, brahmi2018adaptive, lim2019delayed}: hence, we do not introduce any additional constraint while tackling the state-dependent TDE error.
\end{remark}
\subsection{Stability Analysis}
\begin{theorem}
Under Assumptions 1-2 and Property 1, the closed-loop trajectories of the systems (\ref{robotdynamics2}) and (\ref{sys_q}) using the control laws (\ref{theproposedcontrollaw}), (\ref{theproposedcontrollaw_q}) in conjunction with the adaptive laws (\ref{ATRC_p}), (\ref{ATRC}) and design conditions (\ref{E}) and (\ref{E_q}), are Uniformly Ultimately Bounded (UUB).
\end{theorem}
\begin{proof}
See Appendix.
\end{proof}
\begin{figure*}
    \includegraphics[width=\columnwidth]{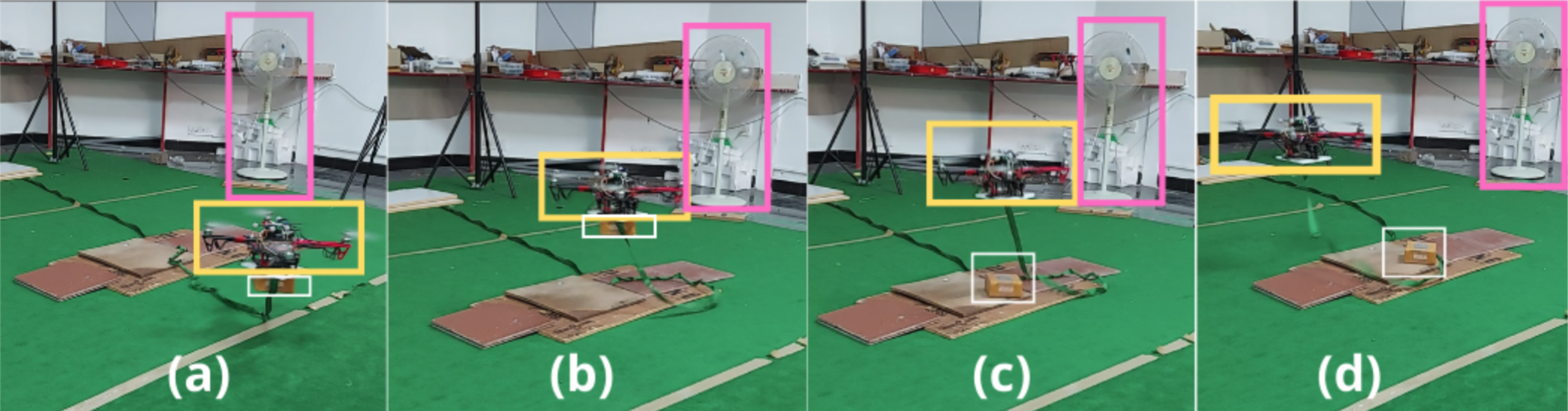}
    \centering
    \caption{Snapshots from the experiment with proposed ARTDE: (a) flying one loop of the trajectory with payload and (b) dropping it at the starting point; (c) initiating the other loop without payload and (d) completing the trajectory.}
    \label{fig:stages}
\end{figure*}

\section{Experimental Results and Analysis}
The proposed ARTDE is tested on a quadrotor setup (Q-450 frame with Turnigy SK3-2826 brushless motors) which uses Raspberry Pi-4 as a processing unit and one electromagnetic gripper (0.03 kg approx.). Excluding the processing unit, gripper and any payload, the quadrotor weighs approx. $1.4$ kg. Optitrack motion capture system (at 60fps) and IMU were used to obtain quadrotor pose and state-derivatives were obtained via fusing these sensor data for the necessary feedback. 
To properly verify the importance of designing state-dependent adaptive control structure, the performance of ARTDE is compared with a conventional adaptive TDE (ATDE) \cite{lee2019adaptive}, and also with a non-TDE adaptive method, adaptive sliding mode control (ASMC) \cite{mofid2018adaptive}, for the sake of completeness. ARTDE mainly differs from ATDE and ASMC in the way its switching gain is adapted to tackle state-dependent uncertainty and the consequent stability analysis. 

\subsubsection{Experimental Scenario and Parameter Selection}
 The objective is the quadrotor should follow an infinity shaped 2-loop path in 3D plane (cf. Figs. \ref{fig:stages}-\ref{fig:3D_Traj}; the height from the ground is purposefully kept relatively small, as it is well-known that near-ground operations are more challenging to control since unknown ground-reaction forces are created by displaced wind from propellers. In addition, a fan is used to create external wind disturbances. The experimental scenario consists of the following sequences (cf. Fig. \ref{fig:stages}): 
\begin{itemize}
    \item The quadrotror starts from the center of the path (where the loops intersect) with a payload ($0.35$kg), it completes one loop with the payload and drops it at its starting position ($t=35$s).
    \item Then, the quadrotor completes another loop without the payload and comes back to the origin.
\end{itemize}
The gripper operation to release the load  is not part of the proposed control design and is operated via a remote signal. For experiment, the control parameters of the proposed ARTDE are selected as: $\bar{ m}=1$ kg, $\bar{J}=0.015I$ (kgm$^2$), ${K}_{2p}=K_{2q} =10$, $  K_{1p}=K_{1q}=5$, $L=0.015$,
${Q_p=Q_q =I}$, $\epsilon_p=\epsilon_q=5\times10^{-5}$, $\underline{\beta}_{ip}=\underline{\beta}_{iq}=0.01$, $\hat{\beta}_{ip}(0)=\hat{\beta}_{iq}(0)=0.01$, $i=0,1$. {For parity in comparison, same values of $\bar{ m}, \bar{J}$, ${K}_{ip},K_{iq}$ and $L=1/60$ are selected for ATDE \cite{lee2019adaptive}}, sliding variables for ASMC \cite{mofid2018adaptive} are selected to be $s_p$ and $s_q$.




\subsubsection{Experimental Results and Analysis} The performances of the controllers are depicted in Figs. \ref{fig:error_exp1}-\ref{fig:error_att_exp1} and further collected in Table \ref{table 1} via root-mean-squared (RMS) error and peak error (absolute value). It is evident from the error plots that both ATDE and ASMC are turbulent while following the trajectory, specifically in the $y$ direction. The quadrotor also swayed before stabilizing after dropping the payload, while the proposed ARTDE could maintain its position after dropping of payload. This can be verified from Table \ref{table 1}, where ARTDE provides more than $30\%$ improved accuracy compared to ATDE and ASMC in $y$ and $z$ directions. These results confirm the benefit of considering state-dependent error-based adaptive control law as opposed to conventional a priori bounded adaptive control structures. 



\begin{figure}[!h]
    \includegraphics[width=3.5in, height= 3.5in]{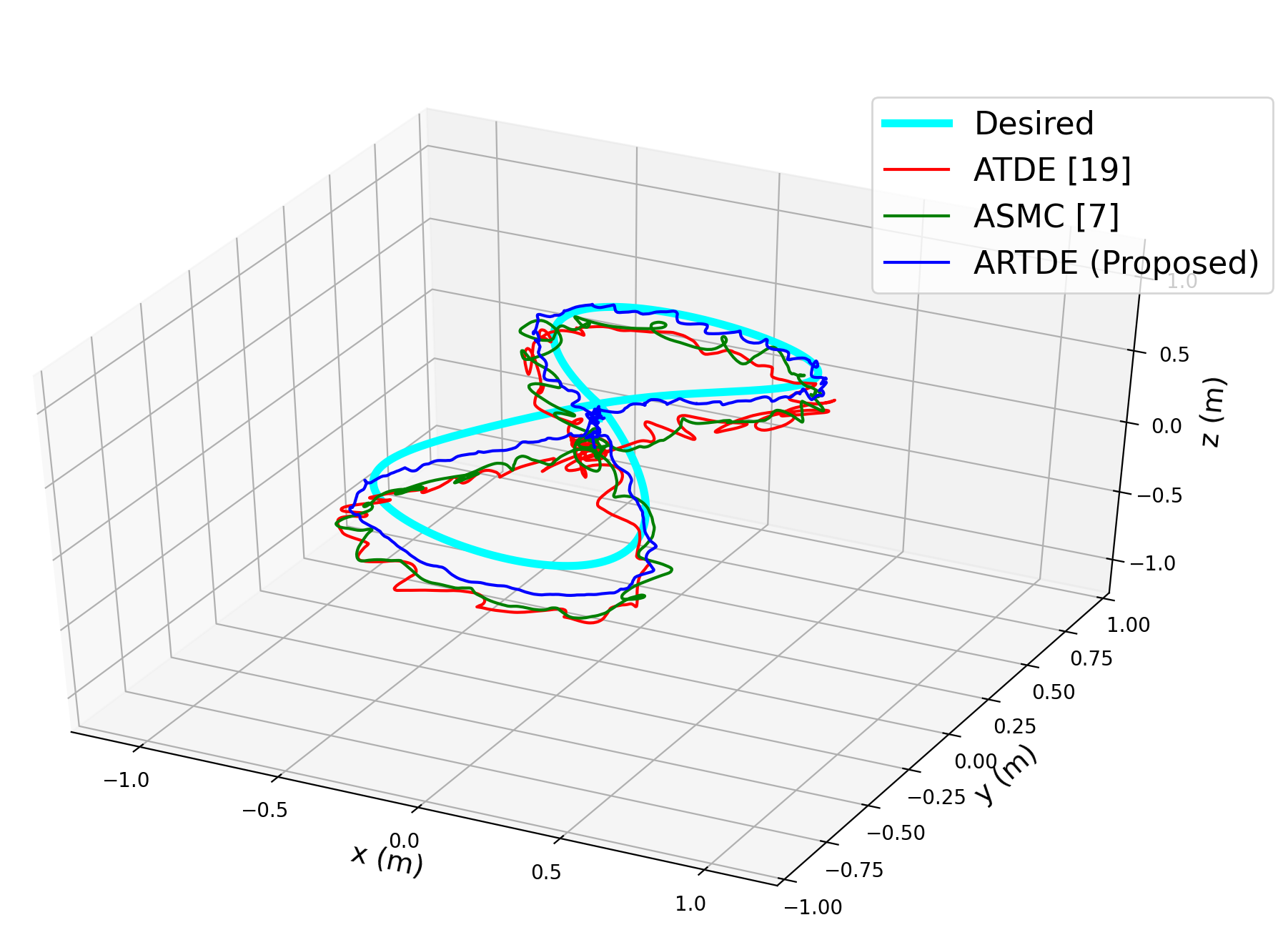}
    \centering
    \caption{Tracking performance of the Infinity-shaped loop.}
    \label{fig:3D_Traj}
\end{figure}
\begin{figure}[!h]
    \includegraphics[width=3.5in, height= 3.5in]{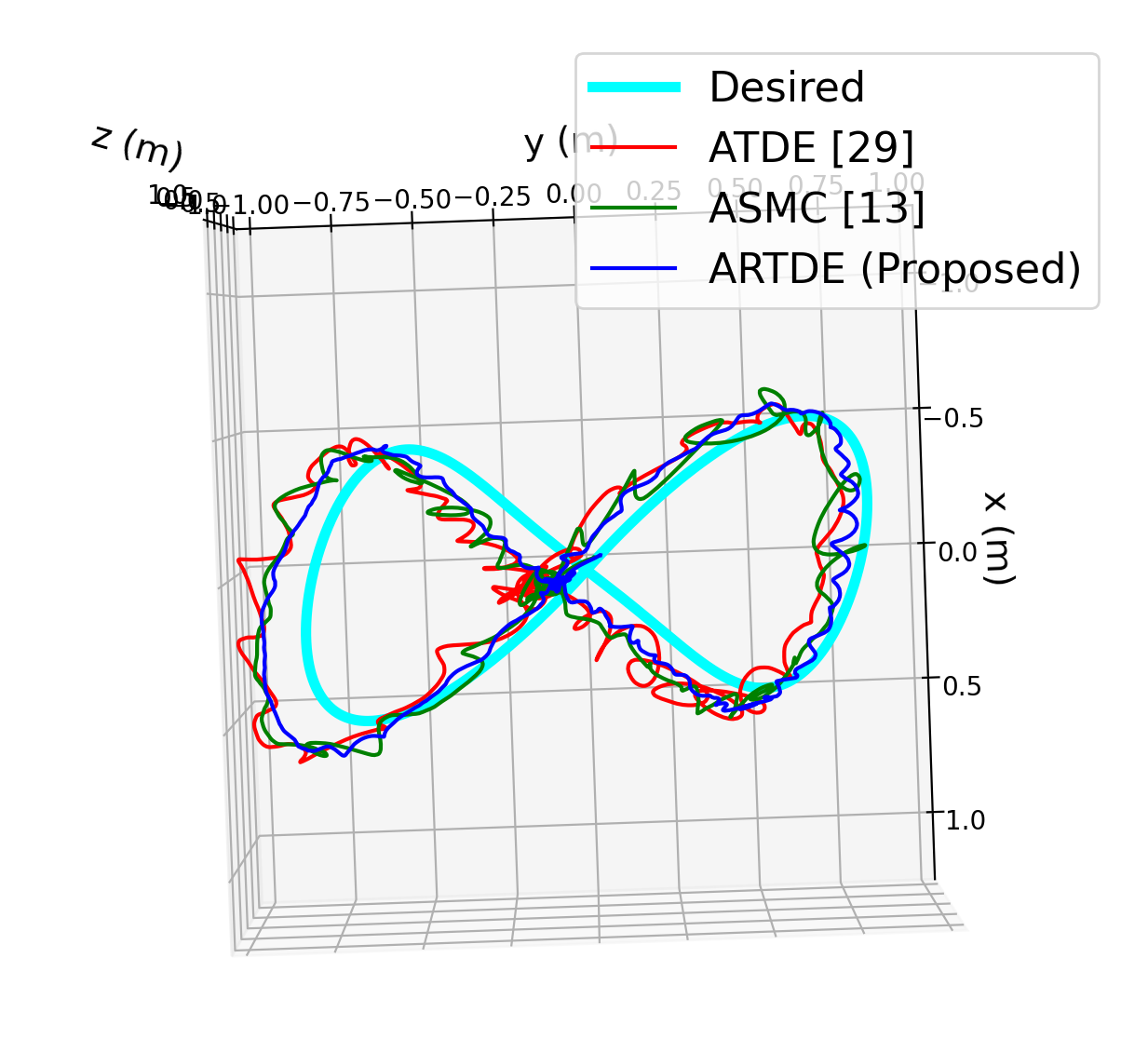}
    \centering
    \caption{Tracking performance of the Infinity-shaped loop in XZ Plane.}
    \label{fig:3D_Traj}
\end{figure}
\begin{figure}[!h]
    \includegraphics[width=3.5in, height= 3.5in]{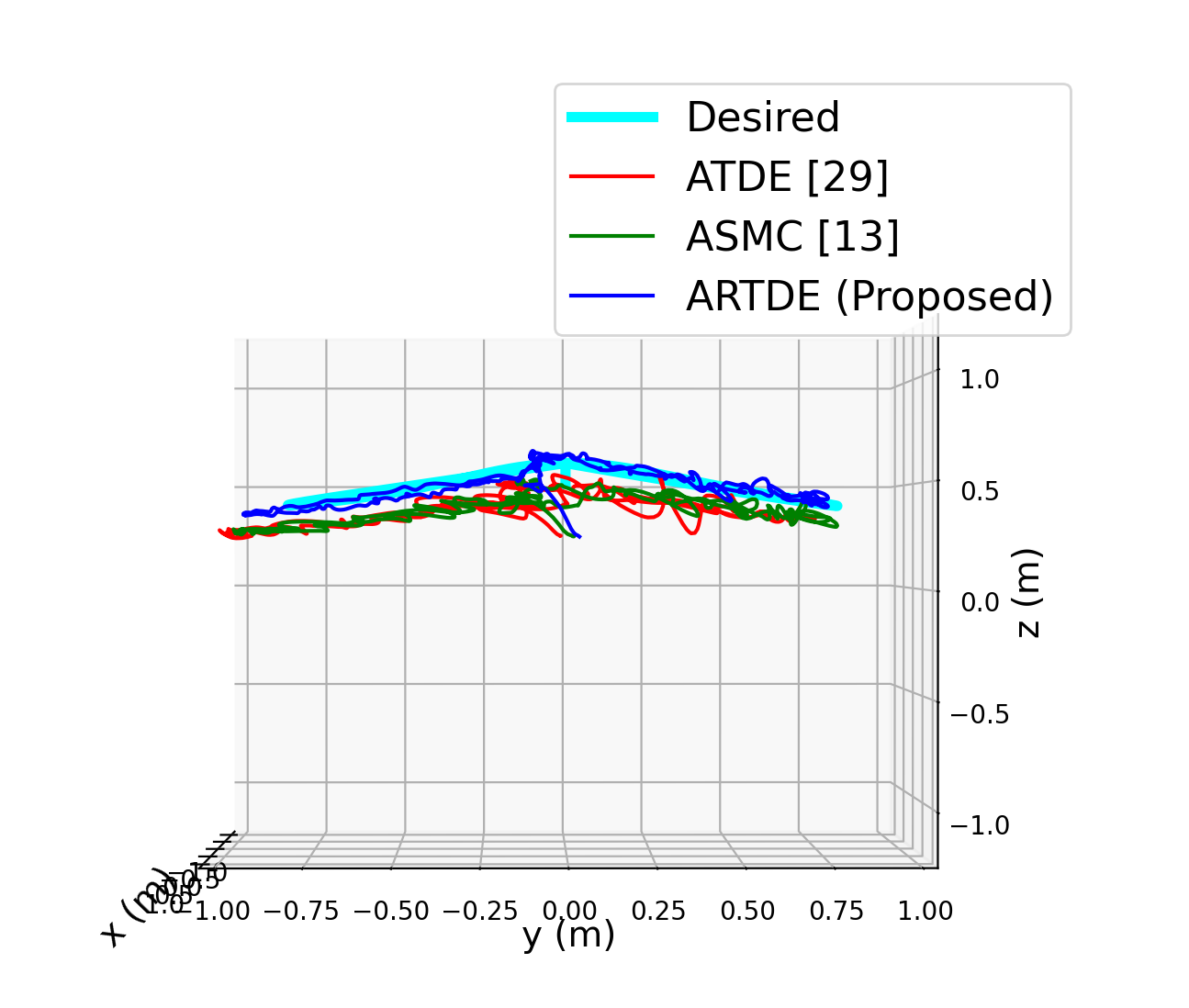}
    \centering
    \caption{Tracking performance of the Infinity-shaped loop in YZ Plane.}
    \label{fig:3D_Traj}
\end{figure}
\begin{figure}[!h]
    \includegraphics[width=\columnwidth]{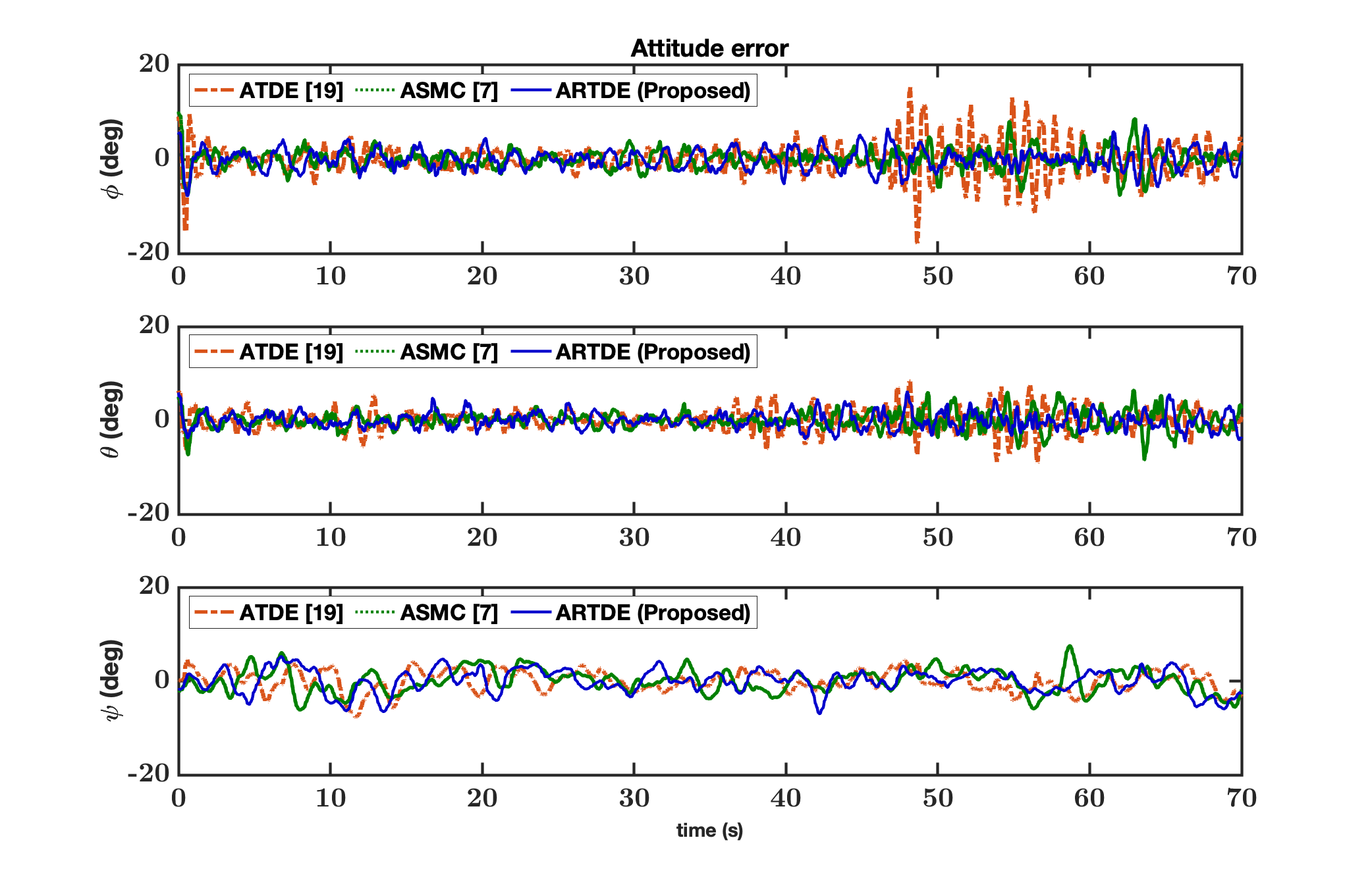}
    \centering
    \caption{Attitude tracking error comparison.}
    \label{fig:error_att_exp1}
\end{figure}

\begin{figure}[!h]
    \includegraphics[width=\columnwidth]{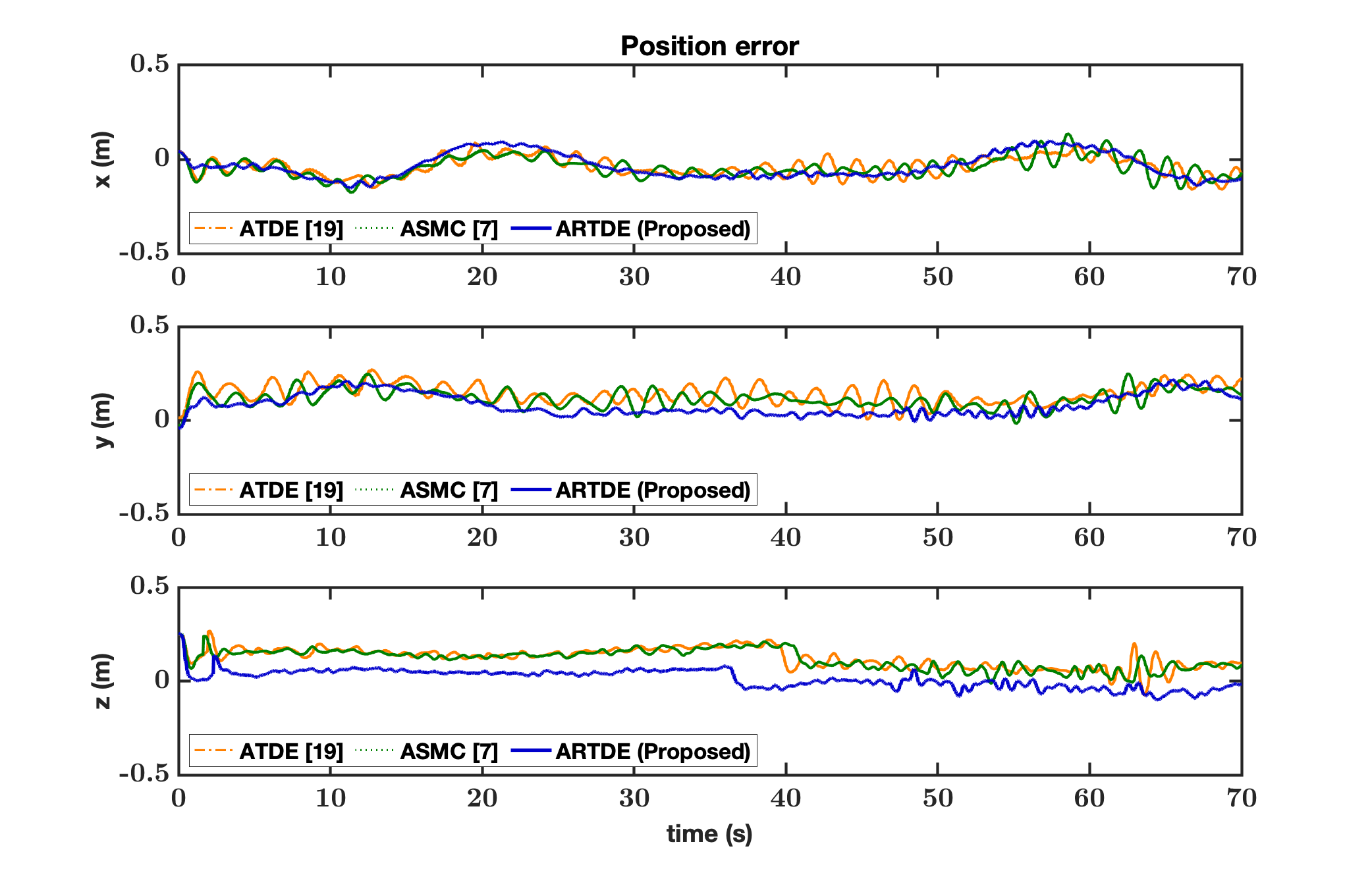}
\centering
    \caption{Position tracking error comparison.}
    \label{fig:error_exp1}
\end{figure}

\begin{table}[!t]
\renewcommand{\arraystretch}{1.0}
\caption{{Performance comparison}}
\label{table 1}
		\centering
{
{	\begin{tabular}{c c c c c c c}
		\hline
		\hline
		& \multicolumn{3}{c}{RMS error (m)} & \multicolumn{3}{c}{RMS error (degree)}  \\ \cline{1-7}
		 Controller & $x$ & $y$  & $z$  & $\phi$ & $\theta$  & $\psi$  \\
		 \hline
		ATDE \cite{lee2019adaptive} & 0.07& 0.15  & 0.13 & 3.37 & 2.28  & 2.22 \\
		\hline
		
		ASMC \cite{mofid2018adaptive}& 0.07 & 0.13 & 0.13 & 2.04 & 1.81 & 2.46 \\
		\hline
		ARTDE (proposed)& { 0.06} & { 0.10}  & {0.04} & 2.12 & 1.72  & 2.50 \\
		\hline
		\hline
		& \multicolumn{3}{c}{Peak error (m)} & \multicolumn{3}{c}{Peak error (degree)}  \\ \cline{1-7}
		 Controller & $x$ & $y$  & $z$  & $\phi$ & $\theta$  & $\psi$  \\
		 \hline
		ATDE \cite{lee2019adaptive} & 0.09 & 0.27  & 0.27 & 15.32 & 8.29  & 6.01 \\
		\hline
		
		ASMC \cite{mofid2018adaptive} & 0.13 & 0.25 & 0.25 & 9.53 & 6.46 & 7.50 \\
		\hline
		ARTDE (proposed)& { 0.09} & { 0.22}  & {0.25} & 7.46 & 6.25  & 5.76 \\
		\hline
		\hline
\end{tabular}}}
\end{table}
\section*{Appendix}
\textbf{Proof of Theorem 1:} The stability analysis of ARTDE is carried out using the following Lyapunov function candidate:
The error dynamics (\ref{error dyn delayed 2}) and (\ref{error dyn delayed 2_q}) can be written as
\begin{align}
\dot{\boldsymbol \xi}_p& =A_p \boldsymbol\xi_p+{B} (\boldsymbol\sigma_p-\boldsymbol\Delta  u_p),
\dot{\boldsymbol \xi}_q=A_q \boldsymbol\xi_p+{B} (\boldsymbol\sigma_q-\boldsymbol\Delta  u_q).\label{error dyn delayed_q}
\end{align}
Positive definiteness of $K_{ip}$ and $ K_{iq}$ $i=1,2$ guarantee that ${A}_p$ and $A_q$ are Hurwitz. The first condition in the adaptive laws (\ref{ATRC_p}), (\ref{ATRC}) reveal that gains $\hat{\beta}_{ip}$ and $\hat{\beta}_{iq}$ increase if they attempt to go below $\underline{\beta}_{ip} $ and $\underline{\beta}_{iq} $ respectively; this yields $\hat{\beta}_{ip}(t) \geq  \underline{\beta}_{ip} $, $\hat{\beta}_{iq}(t) \geq  \underline{\beta}_{iq} $ $\forall t \geq 0$ $\forall i=0,1$. Further, these adaptive laws enumerate the following four possible cases.

\textbf{Case (i):} Both $\dot{\hat{\beta}}_{ip}(t) >0, \dot{\hat{\beta}}_{iq}(t) >0 $ \\
Using the Lyapunov equations ${A}_p^T{U_p+U_pA_p=-Q_p}$ and ${A}_q^T{U_q+U_q A_q=-Q_q}$ the time derivative of (\ref{lyapunov1}) yields
\begin{align}
\dot{{V}} & \leq -(1/2)\boldsymbol\xi_p^{T}{Q_p} \boldsymbol\xi_p+  {s_p}^{T} \lbrace- c_p (s_p/||s_p||)+\boldsymbol \sigma_p \rbrace  \nonumber
-(1/2)\boldsymbol\xi_q^{T}{Q_q} \boldsymbol\xi_q+  {s_q}^{T} \lbrace-c_q (s_q/||s_q||)+\boldsymbol \sigma_q \rbrace  \nonumber\\
&  +\sum_{i=0}^{1}{(\hat{\beta}_{ip}-\beta_{ip})\dot{\hat{\beta}}_{ip}} +{(\hat{\beta}_{iq}-\beta_{iq})\dot{\hat{\beta}}_{iq}} \nonumber\\ 
& \leq - (1/2)(\boldsymbol\xi_p^{T}{Q_p} \boldsymbol\xi_p+\boldsymbol\xi_q^{T}{Q_q} \boldsymbol\xi_q)- c_p || {s_p} || - c_q || {s_q} || + (\beta_{0p} + \beta_{1p}\boldsymbol ||\xi_p ||) || {s_p}||+(\beta_{0q}+ \beta_{1q}|| \boldsymbol \xi_q ||) || {s_q}||  \nonumber\\
& +\sum_{i=0}^{1}\lbrace(\hat{\beta}_{ip}-\beta_{ip})||\xi_p ||^i || {s_p}|| +(\hat{\beta}_{iq}-\beta_{iq})|| \boldsymbol \xi_q ||^i || {s_q}|| \rbrace \nonumber\\
& \leq -(1/2)( \lambda_{\min}(Q_p)||{\boldsymbol \xi_p}||^2+ \lambda_{\min}(Q_q)||{\boldsymbol \xi_q}||^2) \leq 0.  \label{case 1}
\end{align}

 \textbf{Case (ii):} Both $\dot{\hat{\beta}}_{ip}(t) < 0, \dot{\hat{\beta}}_{iq}(t) < 0 $ \\
For this case, the time derivative of (\ref{lyapunov1}) yields
\begin{align}
&\dot{V} \leq -  (1/2)(\boldsymbol\xi_p^{T}{Q_p} \boldsymbol\xi_p+\boldsymbol\xi_q^{T}{Q_q} \boldsymbol\xi_q)- c_p || {s_p} || - c_q || {s_q} || + (\beta_{0p} + \beta_{1p}\boldsymbol ||\xi_p ||) || {s_p}||+(\beta_{0q}+ \beta_{1q}|| \boldsymbol \xi_q ||) || {s_q}||  \nonumber\\
& -\sum_{i=0}^{1}\lbrace(\hat{\beta}_{ip}-\beta_{ip}^{})||\xi_p ||^i || {s_p}|| +(\hat{\beta}_{iq}-\beta_{iq}^{})|| \boldsymbol \xi_q ||^i || {s_q}|| \rbrace \nonumber\\
& \leq -(1/2) (\lambda_{\min}(Q_p)||{\boldsymbol \xi_p}||^2 + \lambda_{\min}(Q_q)||{\boldsymbol \xi_q}||^2)+2 \lbrace (\beta_{0p}^{}+\beta_{1p}^{}|| \boldsymbol \xi_p ||) || s_p || +(\beta_{0q}^{}+\beta_{1q}^{}|| \boldsymbol \xi_q ||) || s_q || \rbrace .\label{case 2}
\end{align}
The second laws of (\ref{ATRC_p}) and (\ref{ATRC}) yield $s_p^T \dot{s}_p \leq 0, s_q^T \dot{s}_q \leq 0$ which imply $|| {s_p} ||, || {s_q} ||, || \boldsymbol \xi_p ||, || \boldsymbol \xi_q ||  \in \mathcal{L}_\infty$ (cf. the relation ${s_p} =  B^T  U_p \boldsymbol \xi_p $, ${s_q} =  B^T  U_q \boldsymbol \xi_q $). Thus, $\exists \varsigma_p, \varsigma_q  \in \mathbb{R}^{+}$ such that $$2 (\beta_{0p}^{}+\beta_{1p}^{}|| \boldsymbol \xi_p ||) || s_p || \leq \varsigma_p , ~2 (\beta_{0q}^{}+\beta_{1q}^{}|| \boldsymbol \xi_q ||) || s_q || \leq \varsigma_q. $$ Further,
The gains $\hat{\beta}_{ip}, \hat{\beta}_{iq} \in \mathcal{L}_{\infty}$ in Case (i) and decrease in Case (ii). This implies $\exists \varpi_p, \varpi_q \in \mathbb{R}^{+}$ such that $$\sum_{i=0}^{1}{(\hat{\beta}_{ip}-\beta_{ip}^{})^2} \leq \varpi_p, ~\sum_{i=0}^{1}{(\hat{\beta}_{iq}-\beta_{iq}^{})^2} \leq \varpi_q$$     Therefore, the definition of $V$ in (\ref{lyapunov1}) yields
\begin{equation}
{V} \leq   \lambda_{\max} ({U_p}) ||{\boldsymbol \xi_p}||^2+ \lambda_{\max} ({U_q}) ||{\boldsymbol \xi_q}||^2 + \varpi_p+\varpi_q. \label{V}
\end{equation}
Using the relation (\ref{V}), (\ref{case 2}) can be written as
\begin{equation}
\dot{V} \leq - \upsilon V  + \varsigma_p+  \varsigma_q + \upsilon (\varpi_p+\varpi_q), \label{V_dot}
\end{equation}
where $\upsilon \triangleq \frac{\min \lbrace  \lambda_{\min}(Q_p),  \lambda_{\min}(Q_q) \rbrace}{\max \lbrace \lambda_{\max} ({U_p}), \lambda_{\max} ({U_q})  \rbrace}$. 

\par \textbf{Case (iii):}  $\dot{\hat{\beta}}_{ip}(t) > 0, \dot{\hat{\beta}}_{iq}(t) < 0 $ \\
Following the derivations for Cases (i) and (ii) we have
\begin{align}
\dot{V} &\leq -(1/2) (\lambda_{\min}(Q_p)||{\boldsymbol \xi_p}||^2 + \lambda_{\min}(Q_q)||{\boldsymbol \xi_q}||^2) +2 (\beta_{0q}^{}+\beta_{1q}^{}|| \boldsymbol \xi_q ||) || s_q ||  \nonumber\\
& \leq  - \upsilon V  +  \varsigma_q + \upsilon \varpi_q. \label{case 3}
\end{align}
\textbf{Case (iv):} $\dot{\hat{\beta}}_{ip}(t) < 0, \dot{\hat{\beta}}_{iq}(t) > 0 $ \\
Following the previous cases yields
\begin{align}
\dot{V} &\leq -(1/2) (\lambda_{\min}(Q_p)||{\boldsymbol \xi_p}||^2 + \lambda_{\min}(Q_q)||{\boldsymbol \xi_q}||^2)+2 (\beta_{0p}^{}+\beta_{1p}^{}|| \boldsymbol \xi_p ||) || s_p ||  \nonumber\\
& \leq  - \upsilon V  +  \varsigma_p + \upsilon \varpi_p.\label{case 4}
\end{align}
The stability results from Cases (i)-(iv) reveal that the closed-loop system is UUB.

%% file: 05_Conclusion.tex
\thispagestyle{empty}
\chapter{Conclusion}

In this thesis, the application of adaptive time-delay based controller to bipedal walking with payload and aerial transportation of payloads via quadrotors was discussed:

\begin{itemize}
    \item \textbf{Bipedal Walking:} In this work, the adaptive controller for
bipedal walking was designed to effectively provide robustness against unmodelled state-dependent constraint forces and impulse forces for all bipedal walking phases. Thus,
the control design and implementation became simpler compared to a multi-modal dynamics based multiple controller paradigm. Via extensive simulations under various forms of disturbances, it was shown that the state-of-art might lead to falling during walking, while the proposed controller could execute the desired walking motion with notable accuracy.

\item \textbf{Aerial Transportation of Payloads via Quadrotors:} In this scenario, artificial delay based adaptive controller for quadrotors was proposed to tackle state-dependent uncertainties. Closed-loop system stability was analytically
verified. The experimental results under uncertain scenario showed significant performance improvements for the proposed controller against state-of-the-art methods.
\end{itemize}

\section{Future Work:}
\begin{enumerate}
    \item \textbf{Bipedal Walking:} The current work has the same controller for both single support and double support phases. In future, we are looking to formulate a suitable controller that could handle the switching in dynamics between single and double support phases. 
    
    \item \textbf{Aerial Transportation of Payloads via Quadrotors:} In this work, the payload is considered to be directly attached to the quadrotor by a gripper or similar mechanism. However, if the size of the payload is big enough to intersect with the rotor plane, then the propwash can destabilize the system \cite{suraj2022introducing}. One possible solution is to suspend the payload via a cable, which, nevertheless, creates a different control challenge by introducing additional unactuated degrees-of-freedom (payload swing angles). In future, an important direction would be to solve such a control challenge in an adaptive setting. 
\end{enumerate}
\vspace{5mm}
\newpage
\begin{center}
 {\Large \bf Publications}
\end{center}

\begin{itemize}
    \item A Bhaskar*, \textbf{S Dantu*}, S Roy, J Lee, S Baldi, "\textbf{Adaptive artificial time delay control for bipedal walking with robustification to state-dependent constraint forces}", \textit{International Conference on Advanced Robotics, 2021} \\
    \textit{* joint first authorship}
    
    \item \textbf{S Dantu}, RD Yadav, S Roy, J Lee, S Baldi, "\textbf{Adaptive Artificial Time Delay Control for Quadrotors under State-dependent Unknown Dynamics}", \textit{ IEEE International Conference on Robotics and Biomimetics, 2022}
\end{itemize}
